\documentclass{article}

\usepackage{arxiv}

\usepackage[utf8]{inputenc} % allow utf-8 input
\usepackage[T1]{fontenc}    % use 8-bit T1 fonts
\usepackage{url}            % simple URL typesetting
\usepackage{booktabs}       % professional-quality tables
\usepackage{amsfonts}       % blackboard math symbols
\usepackage{nicefrac}       % compact symbols for 1/2, etc.
\usepackage{microtype}      % microtypography
\usepackage{mathtools}
\usepackage{hyperref}
\usepackage{tikz}
\usepackage{amssymb}
\usepackage{subcaption}
\usepackage{amsthm}

\newtheorem{thm}{Theorem}[section]
\newtheorem{lem}[thm]{Lemma}

\newtheorem{defn}{Definition}[section]
\newtheorem{corollary}{Corollary}[section]

\title{Learning Smooth and Fair Representations}

\author{
  Xavier Gitiaux \\
  Department of Computer Science\\
  George Mason University\\
  Fairfax, VA 22030 \\
  \texttt{xgitiaux@gmu.edu} \\
   \And
 Huzefa Rangwala \\
  Department of Computer Science\\
  George Mason University\\
  Fairfax, VA 22030 \\
  \texttt{rangwala@cs.gmu.edu} \\
}

\begin{document}
\maketitle

\begin{abstract}
Organizations that own data face increasing legal liability for its discriminatory use against protected demographic groups, extending to contractual transactions involving third parties access and use of the data. This is problematic, since the original data owner cannot ex-ante anticipate all its future uses by downstream users. This paper explores the upstream ability to preemptively remove the correlations between features and sensitive attributes by mapping features to a fair representation space. Our main result shows that the fairness measured by the demographic parity of the representation distribution can be certified from a finite sample if and only if the chi-squared mutual information between features and representations is finite. Empirically, we find that smoothing the representation distribution provides generalization guarantees of fairness certificates, which improves upon existing fair representation learning approaches. Moreover, we do not observe that smoothing the representation distribution degrades the accuracy of downstream tasks compared to state-of-the-art methods in fair representation learning.
\end{abstract}

% keywords can be removed

\section{Introduction}
\label{introduction}
Organizations dealing with data could function as a \textbf{data controller} that determines the purposes and means of processing the data and/or as a \textbf{data processor} that processes the data on behalf of the controller. This distinction has legal ramifications, including challenges related to the discriminatory use of data. For example, the European Union's General Data Protection Regulation (\href{https://gdpr-info.eu/art-4-gdpr/}{GDPR, Article 4}) holds the data controller accountable for the collection, use and disposal of the data, including the responsibility of discriminatory on the basis of sensitive attributes (e.g. racial or ethnic origin, sexual orientation, (\href{https://gdpr-info.eu/recitals/no-71/}{GDPR, Recital 71})). 

However, ex-ante a data controller cannot anticipate what machine learning algorithms a data processor may use to perform its task. Therefore, any contract between a data controller and a data processor is likely to be incomplete and lacks sufficient instructions to guarantee the fairness of any data processor's application. This is problematic since a growing body of evidence has raised concerns about the fairness of machine learning outcomes across a wide range of  applications, including judicial decisions (\cite{ProPublica2016}), face recognition (\cite{pmlr-v81-buolamwini18a}), degree completion (\cite{gardner2019evaluating}) or medical treatment (\cite{pfohl2019creating}).

One promising avenue is for data controllers to limit the data access to fair representations of data instead of the data itself (e.g \cite{madras2018learning}, \cite{Creager2019FlexiblyFR}, \cite{edwards2015censoring}, \cite{pfohl2019creating} or \cite{zemel2013learning}). Fair representation learning seeks to map the original data distribution into a distribution that retains the information contained in the original data, while being statistically independent of sensitive attributes (see Figure \ref{fig:fair_rep}). 
%Usually, the mapping is learned by a neural network (\cite{madras2018learning}, \cite{Creager2019FlexiblyFR}, \cite{edwards2015censoring}) that acts as an information bottleneck to fool  an adversarial fairness auditor (see Figure \ref{fig:fair_rep}).
However, current fair representation learning approaches provide fairness guarantees only against \textit{some} pre-specified data processors (\cite{chouldechova2018frontiers}). 
%In fact, we provide empirical evidence that representations that seem independent of sensitive attributes to a given adversary, would keep leaking information related to sensitive attributes once decoded (see Figure \ref{fig:sw0}, \ref{fig:sw1}). 
\emph{This paper explores conditions on the encoder to generate representation distributions with fairness guarantees that hold for any data processor}. 

\begin{figure}
\centering
\tikzset{every picture/.style={scale=0.775}}
    \begin{tikzpicture}
\small
\draw (0.5, 0) circle (.5cm) node[anchor=center] {$\mathbf{X}$};
\draw[thick, ->, line width=0.5mm] (1, 0.0) -- (1.75, 0);
\draw (1.75, -0.5) rectangle (4.25,0.5) node[pos=0.5] {Encoder $t(\mathbf{X})$};
\draw[->, line width=0.5mm] (3, -1.5) -- (1, 0);
\draw (1.75, -1.5) rectangle (4.25, -2.5) node[pos=0.5] {Decoder $\textstyle g(\mathbf{Z})$};
\draw (5.5, 0) circle (.5cm) node[anchor=center] {$\mathbf{Z}$};
\draw[thick, ->, line width=0.5mm] (4.25, 0) -- (5.0, 0);
\draw[thick, ->, line width=0.5mm] (5, 0) -- (3.0, -1.5);
\draw (6.75, -0.5) rectangle (9.25,0.5) node[pos=0.5] {Auditor $a(\mathbf{Z})$};
\draw[thick, ->, line width=0.5mm] (6, 0) -- (6.75, 0);
\draw (10.5, 0) circle (.5cm) node[anchor=center] {$\mathbf{S}$};
\draw[thick, ->, line width=0.5mm] (9.25, 0) -- (10.0, 0);
\draw[red] (5.5, -2) circle (.5cm) node[anchor=center] {$\epsilon$};
\draw[red, thick, ->, line width=0.5mm] (5.5, -1.5) -- (5.5, -0.5);
\draw[red](5.5, 2) circle (.5cm) node[anchor=center] {$\textstyle h_{i}(Z)$};
\draw[red, thick, ->, line width=0.5mm] (5.5, 0.5) -- (5.5, 1.5);
\draw[red](3.5, 2) circle (.5cm) node[anchor=center] {$\textstyle h_{1}(Z)$};
\draw[red, thick, ->, line width=0.5mm] (5.5, 0.5) -- (3.5, 1.5);
\draw[red](7.5, 2) circle (.5cm) node[anchor=center] {$\textstyle h_{N}(Z)$};
\draw[red, thick, ->, line width=0.5mm] (5.5, 0.5) -- (7.5, 1.5);
\draw[red, loosely dotted, line width=0.5mm](4, 2) -- (5, 2);
\draw[red, loosely dotted, line width=0.5mm](6, 2) -- (7, 2);
\draw[red] (0, 1.5) rectangle (10.5, 2.5) node[pos=0.5]{};
\node[red] at (1, 2.25){Downstream};
\node[red, align=left] at (1, 1.75){Processors:};
\end{tikzpicture}
    \caption{Fair representation learning. Variables are: features $\mathbf{X}$; sensitive attribute $\mathbf{S}$; representation $\mathbf{Z}$. The standard fair representation protocol includes an encoder $t$ that maps $X$ to its representation $Z$; a decoder $g$ that reconstructs $X$ from $Z$; and, an auditor $a$ that measures the statistical dependence between $Z$ and $S$; many downstream data processors $h_{1}$,...., $h_{N}$ that uses the representation $Z$. The contribution of this paper is to add an additive Gaussian white noise (AWGN) channel -- i.e, a convolution step that adds a Gaussian noise $\epsilon$ to $t(\mathbf{X})$ -- so that fairness guarantees can be established for all data processors $h$ using $Z$.}
    \label{fig:fair_rep}
\end{figure}
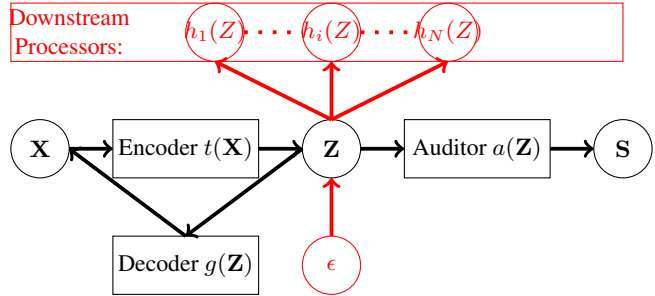

Data controllers would like to produce fairness certificates that measure the potential unfairness of all downstream data processors who access samples from the representation distribution. The question is whether fairness certificates can be approximated by empirical certificates estimated from a finite sample. Our main result shows that for this approximation property of empirical fairness certificates to hold, it is necessary for a measure of information -- the $\chi^{2}$ mutual information -- between feature and representation to be finite. %Intuitively, the issue with diverging $\chi^{2}$ mutual information between features and representations is that any finite sample may fail to capture a positive mass of the representation distribution that leaks information related to sensitive attributes. It means that a fairness auditor may fail to estimates those leakages, but that there is a positive probability for future samples from the representation distribution to expose information related to sensitive attributes. 
Moreover, we prove that a finite $\chi^{2}$ mutual information between feature and representation is a sufficient condition on representation mappings to guarantee a good approximate rate ($O(n^{-1/2})$) of empirical certificates. 
%The connection between finite $\chi^{2}$ mutual information and generalization property of empirical certificates is therefore tight. 

In practice, it is challenging to establish that the $\chi^{2}$ mutual information is finite without knowing the distribution over $\mathcal{X}$. However, we show that an additive Gaussian white noise (AWGN) channel placed after any representation mapping (see Figure \ref{fig:fair_rep}) will bound the $\chi^{2}$ mutual information once the representations have passed through the channel. The channel smoothes the representation distribution by transforming it into a mixture of Gaussian distributions that can be estimated by Monte Carlo integration (\cite{goldfeld2019convergence}). Therefore, a plug-in fairness auditor that relies on estimating the class conditional density functions over the representation space achieves a convergence rate of $O(n^{-1/2})$. Moreover, the AWGN channel offers the possibility to learn high dimensional representations without the need to resort to adversarial auditors that learn to predict the sensitive attribute from samples of the representation distribution (e.g. \cite{madras2018learning} or \cite{edwards2015censoring}). Instead, we approximate the empirical certificate with a differentiable loss that is computed by Monte Carlo integration.
%instead of training an additional neural network to audit the fairness of the representation distribution. 

%As a proof-of-concept, we run experiments on various synthetic  data and fair learning benchmark data 
%Our experiment consists of  assessing the robustness of finite sample demographic parity certificates by testing whether multiple downstream users of samples of the representation distribution  generate a demographic parity smaller than the one estimated by an empirical certificate. 
We empirically find on various synthetic  and fair learning benchmark datasets that an AWGN channel in fair representation learning is sufficient for empirical certificates to upper bound the demographic parity of multiple downstream users that attempts to predict sensitive attributes from samples of the representation distribution. An AWGN channel improves upon existing approaches in adversarial fair representation learning whose fairness guarantees do not extend beyond a set of specific downstream users. Moreover, we did not find strong evidence that obtaining good approximation rates for empirical certificates comes at the cost of significantly degrading the accuracy-fairness trade-off of downstream predictive tasks. 
%We demonstrate that the AWGN channel is necessary and sufficient for finite sample certificates to characterize the demographic parity for all downstream users of the generated representations. We then apply our framework to real-world data and show that by smoothing the representation distribution with a Gaussian convolution, we obtain robust finite sample certificates. Moreover, we find empirically that the robustness of demographic parity certificates does not degrade the accuracy-fairness trade-off that results from imposing fairness constraints on representation distributions. 

%\paragraph{Contributions}
%Our contributions are as follows:
%\begin{itemize}
 %   \item We show that finite sample demographic parity certificates do not provide fairness guarantees for all future use of the representation distribution, unless the encoder does not pass any information between the data and the representation.
%    \item We prove on the other hand that adding a noisy channel allows to estimate fairness certificates that converge at a $O(n^{-1/2})$ rate to true demographic parity certificate.
 %   \item We introduce a procedure that uses robust fairness certificates to learn representation distributions with fairness guarantees that are valid for all downstream uses of the data.
%\end{itemize}

\textbf{Related work.} A growing literature explores the potential adverse implications that machine learning algorithms 
%involved in sensitive decision-making processes 
might have on protected demographic groups (e.g individuals self-identified as Female or African-American) (\cite{chouldechova2018frontiers} for a review). Many contributions seek to define fairness criteria either at the group or individual level (\cite{dwork2012fairness}) and then, impose a fairness penalty into their classification algorithm (e.g. \cite{agarwal2018reductions}, \cite{kim2018fairness}, \cite{kearns2018preventing}) or audit for a specific criteria (e.g \cite{feldman2015certifying}, \cite{Gitiaux2019mdfaMF}). In this paper, we side-step the important discussion on what fairness criteria to choose from (\cite{kleinberg2016inherent}), but investigate whether a data can be transformed so that any future use will meet a pre-specified criteria. Our results focus on demographic parity (\cite{dwork2012fairness}), but can be readily extended to many other group level criteria, including equalized odds and equal opportunity (\cite{hardt2016equality}). 

Existing pre-processing methods to mitigate unfair data use include sampling and reweighting (e.g. \cite{calders2013unbiased}, \cite{gordaliza2019obtaining}), optimization procedures  to learn a data transformation that both preserve utility and limit discrimination (e.g. \cite{calmon}), and representation learning (e.g. \cite{zemel2013learning}). Representation learning seeks to encode the data while removing correlations between features and sensitive attributes. Recent developments in adversarial learning for generative modeling (see \cite{kurach2018largescale} for a survey) or domain adaptation (e.g. \cite{ganin2016domain}) have spurred an interest in training a data encoder to generate a representation of the data and fool a neural network that attempts to predict sensitive attributes from samples of the representation distribution (e.g. \cite{edwards2015censoring}, \cite{madras2018learning}, \cite{zhang2018mitigating} or \cite{xu2018fairgan}). An alternative approach is to disentangle sensitive attributes from features by passing the data through an information bottleneck (\cite{louizos2015variational} or \cite{Creager2019FlexiblyFR}). 

Our contribution to the fair adversarial learning literature is to explore conditions so that the learned representation offers fairness guarantees against adversaries that do not necessarily belong to the same class as the adversary used during the training of the encoder. 
%\cite{beutel2017data} show that the data distribution can affect the fairness achieved by adversarial learning. 
\cite{madras2018learning} and \cite{oneto2019learning} explore empirically whether representations that achieve demographic parity for a specific downstream task generalize to new tasks in terms of accuracy and fairness. We extend their work by showing theoretically and empirically that introducing an AWGN channel in fair representation learning offers generalization guarantees to all future tasks.

Similar to our approach, the differential privacy literature relies on noise injection to guarantee that two neighboring datasets are indistinguishable (\cite{dwork2014algorithmic}).  However, there is an important difference between these approaches. In the context of differential privacy, indistinguishability is only obtained by adding Gaussian/Laplacian noise. In our fairness context, for a finite sample, statistical hiding comes from learning representations subject to a demographic parity constraint; the injection of Gaussian noise is only a means to generalize the statistical hiding property to the infinite sample regime. Of interest is whether this two-step approach has merits in a privacy setting.

\section{Certifying Fair Representations}
\label{sec: 2}
\subsection{Background}

Consider a data controller who wants to release  samples from a distribution $\mu$ over  $\mathcal{X}\times \mathcal{S}$ with features in $\mathcal{X}\subset[0, 1]^{D}$   and sensitive attributes in $\mathcal{S}$. Although our setup can be extended to richer spaces of sensitive attributes, we focus here on binary sensitive attributes and assume that $\mathcal{S}=\{0, 1\}$. 

A transformation $t$ that maps the features space $\mathcal{X}$ into a representations space $\mathcal{Z}\subset\mathbb{R}^{d}$ induces a distribution $\mu_{t}$ over $\mathcal{Z}\times \{0, 1\}$:  $\mu_{t}(A)=\mu\left(\{x\in \mathcal{X}| t(x)\in A\}\right)$ for any $A\subset \mathcal{Z}$. 
%The dimension $d$ of the representation space $\mathcal{Z}$ is likely to be smaller than the dimension $D$ of the feature space $\mathcal{X}$ if the representation mapping is used as an information bottleneck to encode the data.

The data controller's objective is to obtain a representation mapping $t$ that minimizes the statistical dependence between representation $Z$ and sensitive attribute $S$. Therefore, for any test $f:\mathcal{Z}\rightarrow \{0, 1\}$ that decides whether the class conditional distributions $\mu_{t}^{0}=P(Z|S=0)$ and $\mu_{t}^{1}=P(Z|S=1)$ are identical, the data controller would like to minimize the discrepancy
\begin{equation}
    \Delta(f, t) \triangleq |E_{z\sim \mu_{t}^{1}}[f(z)] - E_{z\sim \mu_{t}^{0}}[f(z)]|,
\end{equation}
where we make the dependence of $\Delta$ on representation mapping $t$ explicit. In the context of fair machine learning, the test function $f$ is either an auditor used by the data controller to estimate the statistical dependence between $Z$ and $S$ (function $a$ in Figure \ref{fig:fair_rep}); or, a classifier used by a data processor (function $h$ in Figure \ref{fig:fair_rep}) and $\Delta(f, t)$ then measures the demographic parity of $f$ (see \cite{hardt2016equality}):

\begin{defn}{{\bf Demographic parity}}
Consider a representation distribution $\mu_{t}$ induced by a representation mapping $t:\mathcal{X}\rightarrow\mathcal{Z}$. A classifier $f:\mathcal{Z}\rightarrow \{0,1\}$ used by a data processor satisfies $ \delta-$ Demographic Parity on $\mu_{t}$ if and only if $\Delta(f, t) \leq \delta.$
\end{defn}
%has many fairness criteria\footnote{This is true to the extent that some criteria are incompatible \cite{kleinberg2016inherent}: for example, balanced positive, negative rates and calibration within groups cannot be achieved at the same time.} (\cite{hardt2016equality}) to choose from. In this paper, we focus on one definition -- demographic parity -- but our results and algorithms can be extended to other group level definition of fairness including equalized odds and equalized opportunities (\cite{hardt2016equality}). 
Since the data controller does not know ex-ante which classifier data processors will use, she has to construct a mapping $t$ such that all classifiers $f:\mathcal{Z}\rightarrow \{0,1\}$ satisfy $ \delta-$ demographic parity  on $\mu_{t}$ for some pre-specified $\delta > 0$. A demographic parity certificate is therefore an upper bound on the demographic disparity of any classifiers that access samples from the representation distribution $\mu_{t}$.
 
% The data controller will then contract with downstream data processors to allow access to samples from  $\mu_{t}$ instead of $\mu_{\mathcal{X}}$. The advantage is that the data controller will be certain that all data processors will satisfy $\epsilon-$ demographic parity.

\begin{defn}{\bf Demographic Parity Certificate} Let $\delta \geq 0$. A representation space $(\mathcal{Z}, \mu_{t})$ can be certified with $\delta-$ demographic parity if and only if
\begin{equation}
    \label{eq: cert}
    \Delta^{*}(t)\triangleq \sup\limits_{f:\mathcal{Z}\rightarrow \{0, 1\}}\Delta(f, t) \leq \delta.
\end{equation}
\end{defn}

%An $\Delta^{*}(t)-$ demographic parity certificate is an upper bound on the demographic parity that can be generated by a classifier using samples from $\mu_{t}$: there is no classifier $f:\mathcal{Z}\rightarrow \{0, 1\}$ with $\Delta(f, t) > \Delta^{*}(t)$. 

To construct  a representation mapping certified with $\Delta^{*}(t)-$ demographic parity, the  data controller needs to evaluate the supremum over all test functions/auditors $f_{n}$ that are constructed on the basis of a finite sample $\mathcal{D}_{n}=\{(x_{i}, s_{i})\}_{i=1}^{n}$. Let $\mathcal{F}_{n}$ denote the set of all auditors $f_{n}: \mathcal{Z} \times (\mathcal{Z} \times \{0, 1\})^{n}\rightarrow \{0, 1\}$ constructed from a sample of size $n$.

%and to measure $\Delta(f_{n}, t)$. An auditor is thus a test function $f_{n}: \mathcal{Z} \times (\mathcal{Z} \times \{0, 1\})^{n}\rightarrow \{0, 1\}$ that depends on the size of the data used to learn it. 

\begin{defn}{\bf Empirical Demographic Parity Certificate} Let $n\geq 1$ and $\delta \geq 0$. A representation space $(\mathcal{Z}, \mu_{t})$ is certified with an empirical $\delta-$ demographic parity certificate if and only if
\begin{equation}
    \label{eq: cert_n}
   \Delta_{n}(t)\triangleq \sup\limits_{f_{n}:\in \mathcal{F}_{n}}\Delta(f_{n}, t) \leq \delta.
\end{equation}
\end{defn}
This paper investigates how to choose a representation mapping $t:\mathcal{X}\rightarrow{Z}$ so that empirical certificates are good approximations of the true demographic parity certificate, i.e. $\Delta_{n}(t)$ approximates well $\Delta^{*}(t)$. Approximation properties of empirical certificates are important for a data controller to anticipate the demographic parity of a downstream processor who uses fresh samples obtained after $t$ has been constructed. 

Since the data controller cannot constrain the data distribution over $\mathcal{X}\times \{0, 1\}$, we are looking for distribution-free approximation rates. In general, distribution-free rates do not exist (\cite{devroye2013probabilistic}, ch. 7). But, in our setting, the data controller has some control over the representation distribution via $t$. In fact, the approximation $\Delta^{*}(t) -\Delta_{n}(t)$ depends on how much information in $X$ is encoded by $t$ in $Z$. If $t$ randomly maps $\mathcal{X}$ to $\mathcal{Z}$, the data controller can certify $\mu_{t}$ with $0-$ demographic parity, but $\mu_{t}$ is useless to downstream data processors. The data controller trades-off representation demographic parity with information by learning a representation mapping $t:\mathcal{X}\rightarrow \mathcal{Z}$ and a decoder function $g:\mathcal{Z}\rightarrow \mathcal{X}$ that solves the following fair empirical representation problem
\begin{equation}
    \label{eq: fair_rep_prob}
    \min_{t, g} \mathcal{L}_{rec}(g, t, \mathcal{D}_{n})  \text{ subject to }  \Delta_{n}(t)\leq \delta,
\end{equation}
where $\delta >0$ is a pre-specified demographic parity threshold and $\mathcal{L}_{rec}$ is a reconstruction loss whose choice depends on the data. 

\subsection{Necessary Condition}
This section identifies a necessary condition on deterministic representation mapping $t$ for the induced empirical demographic parity certificate to approximate $\Delta^{*}(t)$ well. The necessary condition bounds the amount of information measured by the $\chi^{2}$ mutual information between feature $X$ and representation $Z$:
\begin{equation}
    I_{\chi^{2}}(X, Z) \triangleq E_{x}E_{z}\left(\frac{\mu_{t}(z) - \mu_{t}(Z|X=x)}{\mu_{t}(z)}\right)^{2}.
\end{equation}
The $\chi^{2}$ mutual information relies on a statistical distance, the $\chi^{2}-$divergence
\begin{equation}\nonumber
\chi^{2}(Z, Z|X)=\int_{z}\left(dP(Z|X)/dP(Z) -1\right)^{2}dP(Z)
\end{equation}
to average the distance between $Z$ and $Z|X=x$ for $x\in \mathcal{X}$. It has been used in information theory to estimate the information that flows through a neural network (see \cite{goldfeld2019convergence}). In the context of fair representation learning, we find that empirical demographic parity certificates cannot provide good approximations of the representation's true demographic parity if the $\chi^{2}$ input-output mutual information is large:

\begin{thm}
\label{th: 1}
Let $n\geq 1$. Consider a representation function $t:\mathcal{X}\rightarrow\mathcal{Z}$. Then, 
\begin{equation}
    \label{eq: chi_necessary}
    \inf\limits_{f_{n}\in \mathcal{F}_{n}}\sup\limits_{\mu} E_{\mathcal{D}_{n}}|\Delta^{*}(t) - \Delta(f_{n}, t) |- \left(1 -\frac{1}{I_{\chi^{2}}(X,Z)}\right)^{n}\geq 0. 
\end{equation}
\end{thm}
We prove the result in Theorem \ref{th: 1} for any representation function and conjecture, but have not proved it, that it still holds true for any representation mapping. 
%Intuitively, a representation function that encodes more information from the features space to the representations space will generate representation distributions that mirror singularities -- if any -- present in the distribution over the features. Since estimating $\Delta^{*}(t)$ is akin to estimating the balanced error rate of predicting the sensitive attribute from representations $Z$ (see \cite{feldman2015certifying} or \cite{mcnamara2017provably}), encoding more information in $Z$ will expose empirical certificates to the lack of distribution-free finite sample approximation that are well known in pattern recognition (see \cite{devroye2013probabilistic}, ch 7). Large $I_{\chi^{2}}(X, Z)$ means that on average the atoms of $p(Z|x)$ are far apart from each other and thus, so that for a finite sample $\mathcal{D}_{n}$, there is an large number of atoms which are not in $\mathcal{D}_{n}$, but which count for part of the mass of the representation distribution $\mu_{t}$.
Encoding more information of $X$ in $Z$ exposes the representation distribution $\mu_{t}$ to mirroring distributions over $\mathcal{X}$ with heavy tails. Intuitively, $\mu_{t}$ is a (possibly infinite) mixture of conditional distributions $P(Z|X=x)$ for $x\in \mathcal{X}$ and $I_{\chi^{2}}(X, Z)$ measures an average distance between those conditional distributions. As $I_{\chi^{2}}(X, Z)$ increases, the conditional distributions $P(Z|X=x)$ become far apart for a growing mass of $x\in \mathcal{X}$. It generates a representation distribution too complex for a finite sample to represent it and for an auditor $f_{n}$ to detect all the correlations between representation and sensitive attribute. 

Theorem \ref{th: 1} implies a trade-off between the information passed from feature to representation and the approximation rate of empirical demographic parity certificates: 

\begin{corollary}
\label{cor: rates}
With the notations from Theorem \ref{th: 1}, suppose that 
\begin{equation}
    \nonumber
    \inf\limits_{f_{n}\in \mathcal{F}_{n}}\sup\limits_{\mu} E_{\mathcal{D}_{n}}|\Delta^{*}(t) - \Delta(f_{n}, t) | \leq \epsilon_{n},
\end{equation}
then for all distributions over the feature space $\mathcal{X}$,
\begin{equation}
    I_{\chi^{2}}(X, Z) \leq \frac{1}{1-\epsilon_{n}^{\frac{1}{n}}}.
\end{equation}
\end{corollary}

 The smaller the approximation rate $\epsilon_{n}$ is, the smaller is the upper bound on the $\chi^{2}$-mutual information between $X$ and $Z$.  For the approximation rate of $\Delta^{*}(t) - \Delta(f_{n}, t)$ to be $O(n^{-s})$ for some $s > 0$, it is necessary for the $\chi^{2}$ mutual information between feature and representation to be bounded above by $O(n/(s\ln(n))$ for \emph{all} distributions over $\mathcal{X}$. On the other hand, representation functions $t$ for which the $\chi^{2}$ mutual information is infinite for some distribution over the features space, never guarantee a meaningful approximate rate between $\Delta^{*}(t)$ and $\Delta_{n}(f_{n}, t)$ for any auditor $f_{n}$: 

\begin{corollary}
\label{cor: 1}
Let $n\geq 1$. Consider a representation function $t:\mathcal{X}\rightarrow\mathcal{Z}$. Suppose that there exists a distribution over $\mathcal{X}$ such that $I_{\chi^{2}}(X, Z)=\infty$. Then, 
\begin{equation}
    \label{eq: chi_necessary_div}
    \inf\limits_{f_{n}\in \mathcal{F}_{n}}\sup\limits_{\mu} \Delta^{*}(t) - \Delta(f_{n}, t)\geq 1. 
\end{equation}
\end{corollary}

 \textbf{Examples: }The results in corollary \ref{cor: rates} and \ref{cor: 1} imply that empirical certificates of representation distributions induced by many common encoders do not have meaningful approximation rates:
 \begin{itemize}
    \item Suppose that $t$ is injective from $\mathbb{R}^{D}$ to $\mathbb{R}^{d}$. Then, there exists a distribution over $\mathcal{X}\times \{0, 1\}$ such that $I_{\chi^{2}}(X, Z)=\infty$ and thus, $\Delta^{*}(t)=1$, but $\Delta(f_{n}, t)=0$ for all auditing functions $f_{n}$.
    \item Suppose that $|\{t(x)| x\in\mathcal{X}\}| \geq n/(\ln(n))^{\alpha}$, for some $\alpha < 1$. Then, the approximation rate of $\Delta(f_{n}, t)$ for all auditing functions $f_{n}$ is $\omega(n^{-s})$ for any $s>0$.
\end{itemize}

%One implication of theorem \ref{th: 1} is that a solution of empirical fair representation problem \eqref{eq: fair_rep_prob} might satisfy at best a $1-$ demographic parity certificate ($\Delta^{*}(t) \geq 1)$ regardless of how stringent is the demographic parity constraint (as measured by $\delta$).  

%\paragraph{Examples}
%To illustrate the intuition behind the results in theorem \ref{th: 1}, we draw $b$ uniformly from $(0,1)$, and choose features $X$ uniformly drawn from $\mathcal{X}=\{0, 1, ..., K\}$ for some $K>0$ and a sensitive attribute $S=\{0, 1\}$ such that for $X\in [k * i, (i+1)*k)$, $S$ is the $i^{th}$ binary expansion of $b$. The parameter $k\leq K$ measures the size of the bins where the sensitive attribute is constant. In our first experiment, we consider a representation function $t:\mathcal{X}\rightarrow \mathcal{Z}$ such that for $X\in [k * i, (k + 1)*i)$, $t(X)= i$. That is, $t$ aggregates all features in a bin $[k * i, (k + 1)*i)$ to a single and unique representation $i$. Since $S$ is only a function of the bin to which $X$ belongs to and since $t$ preserves the bin structure of the features space, the sensitive attribute and perfectly predictable and $\Delta^{*}(t)=1$. Moreover, the $\chi^{2}-$ mutual information between $X$ and $Z$ is analytically tractable with $I_{\chi^{2}}(Z, X)=\frac{K}{k} - 1$. 

\subsection{Sufficient Condition}
This section shows that a finite $\chi^{2}$ mutual information between feature and representation for all distributions over $\mathcal{X}$ is a sufficient condition for empirical demographic parity certificates to converge at a $O(n^{-1/2})$ rate.

\begin{thm}
\label{th: 2}
Let $n\geq 1$. Consider a representation mapping $t:\mathcal{X}\rightarrow\mathcal{Z}$. Then, if $\mathcal{F}_{n}$ denotes the set of all auditors $f_{n}: \mathcal{Z} \times (\mathcal{Z} \times \{0, 1\})^{n}\rightarrow \{0, 1\}$, if $n_{s}=|\{ i| s_{i}=s\}|$, 
\begin{equation}
    \label{eq: chi_sufficient}
    \begin{split}
    \inf\limits_{f_{n}\in \mathcal{F}_{n}}\sup\limits_{\mu} E_{\mathcal{D}_{n}}|\Delta^{*}(t) - \Delta(f_{n}, t)| & \\
    - 2\sum_{s=0, 1}n_{s}^{-1/2} \sqrt{I_{\chi^{2}}(X, Z|S=s)}\leq 0 & \\
    \end{split}
\end{equation}
\end{thm}

A finite $\chi^{2}$ mutual information between $X$ and $Z$ implies that $P(Z)$ and $P(Z|X)$ are close in the sense of  the $\chi^{2}$ divergence and thus by sampling representations from $P(Z|X)$, we have a non-zero probability to sample all the atoms that can form the representation distribution $\mu_{t}$ and thus to detect all the dependence between representations and sensitive attributes.

\subsection{\texorpdfstring{$\chi^{2}$}{e} versus Classic Mutual Information}
Our results in Theorems \ref{th: 1} and \ref{th: 2} highlight the connection between the $\chi^{2}$ mutual information and the approximation rate of empirical certificates. A similar result cannot be obtained with the classic mutual information $I_{Sh}(X, Z)$ that is based on Shannon entropy. 

To demonstrate this point, we construct the following distribution $\mu$ over $\mathcal{X}\times\{0, 1\}$. Features are uniformly distributed over $[0, 1]$ and $t(x)=i$ for $x\in [1/i, 1/(i+1))$ and $i>0$. For each $i>0$, the sensitive attribute is constant over $[1/i, 1/(i+1))$ and equal to $1$ with probability $1/2$. We show in the appendix that $I_{Sh}(X, Z) < \ln(6)/3 + 3$, but $I_{\chi^{2}}(X, Z)=\infty$. Since the sensitive attribute $S$ is a deterministic function of the representation $Z=t(X)$, $\Delta^{*}(t)=1$. But, for a finite sample of size $n$, $E_{\mathcal{D}_{n}}\Delta(f_{n}, t)$ is zero for all auditors $f_{n}$, despite $I_{Sh}(X, Z) <\infty$. 

%Hence, our use of the $\chi^{2}-$ mutual information is not accidental: there is no equivalent of Theorem \ref{th: 2} with the classic mutual information instead of the $\chi^{2}$ one.

\section{Smooth and Fair Representations}
The previous section suggests restricting the fair representation problem \eqref{eq: fair_rep_prob} to representation mappings for which the $\chi^{2}$mutual information between feature and representation is finite for all distributions over $\mathcal{X}$. 
%The challenge is that this finitude has to be true for all features distributions.  
In this section, this sufficient condition is met by introducing an additive Gaussian white noise (AWGN) channel after the encoder $t$.

\subsection{Convergence of Smoothed Empirical Certificate}
For any representation mapping $t:\mathcal{X}\rightarrow\mathcal{Z}$, we denote $t_{\sigma}$ the convolution of $t$ with a Gaussian noise $\mathcal{N}(0, \sigma^{2}I_{d})$: $t_{\sigma}(X)= t(X) + noise$, with $noise\sim \mathcal{N}(0, \sigma^{2}I_{d})$. The convolved representation $Z_{\sigma}=Z + noise$ has a distribution denoted $\mu_{t*\sigma}$. The convolution 
%makes the representation mapping a stochastic channel and the $\chi^{2}-$mutual information between features $X$ and its representations $Z_{\sigma}$ is finite. Intuitively, an AWGN channel 
smoothes the representation distribution by making $P(Z_{\sigma}|X)$ a Gaussian whose support covers the support of the representation distribution $P(Z_{\sigma})$ and thus, guarantees that samples from different conditional distributions $P(Z_{\sigma}|X=x)$ are not too far away. 
\begin{thm}
\label{th: 3}
Let $\sigma > 0$ and $n\geq 1$. For all representation mapping $t:\mathcal{X}\rightarrow \mathcal{Z}$ and for any distribution over $\mathcal{X}$, if $||t||_{\infty}\triangleq\sup_{x\in\mathcal{X}}||t(x)||_{2}$ , then for $s\in \{0, 1\}$
\begin{equation}
    I_{\chi^{2}}(X, Z|S=s) \leq \exp\left(\frac{||t||_{\infty}^{2}}{\sigma^{2}}\right) <\infty.
\end{equation}
Therefore,  
\begin{equation}
\begin{split}
   \inf\limits_{f_{n}\in \mathcal{F}_{n}} \; \sup\limits_{\mu}E_{\mathcal{D}_{n}}[\Delta^{*}(t_{ \sigma}) - \Delta(t_{\sigma}, f_{n})] & \\ \leq  2\exp\left(\frac{||t||_{\infty}^{2}}{2\sigma^{2}}\right)(n_{0}^{-1/2} + n_{1}^{-1/2}).& \\
   \end{split}
\end{equation}
\end{thm}
The upper bound in Theorem \ref{th: 3} does not depend on the dimensions $d$ of the representation space $\mathcal{Z}$, but only on $n^{-1/2}$ and on the ratio $||t||_{\infty}/\sigma$ that can be interpreted as a signal-to-noise ratio in the AWGN channel. Larger values of $||t||_{\infty}$ increase the variance of $Z$ and thus require larger noise $\sigma$ to keep the conditional distribution $P(Z_{\sigma}|X)$ close to the distribution $P(Z_{\sigma})$. The bound is only meaningful if $||t||_{\infty} <\infty$, which holds, for example, if the features space is bounded and $t$ is a continuous mapping. 

Both Theorems \ref{th: 2} and \ref{th: 3} rely on  a plug-in auditor that first estimates the class-conditional densities $\mu_{t*\sigma}^{0}$ and $\mu_{t*\sigma}^{1}$. From a sample $\mathcal{D}_{n}=\{(x_{i}, s_{i})\}_{i=1}^{n}$, we construct an empirical estimate of $\mu_{t* \sigma}$ over $\mathcal{Z}\times \{0, 1\}$ as  
\begin{equation}
\label{eq: mu_n2}
    \mu_{n, \sigma}(z, s) = \displaystyle\sum_{i=1, s_{i}=s}^{n}P(z| X=x_{i})
\end{equation}
with $P(.| X=x_{i})\sim \mathcal{N}(t_{n}(x_{i}), \sigma I_{d})$. Our plug-in auditor $f_{n}^{plug}$ compares $\mu_{n, \sigma}(z, 0)$ to $\mu_{n, \sigma}(z, 1)$:
\begin{equation}
f_{n}^{plug}(z) =
\begin{cases}
0 & \text{if } \mu_{n, \sigma}(z, 0) \geq \mu_{n, \sigma}(z, 1) \\
1 & \text{otherwise.} 
\end{cases}
\end{equation}
Since we obtain the upper bounds in Theorems \ref{th: 2} and \ref{th: 3} with the plug-in auditor $f_{n}^{plug}$, we can guarantee that the representation demographic parity is within $O(n^{-1/2})$ of the empirical certificate signed by the plug-in auditor.

\subsection{Learning Fair Representation}
\begin{figure*}
\vskip 0.2in
 \centering
     \begin{subfigure}{0.24\textwidth}
    \includegraphics[width=\linewidth]{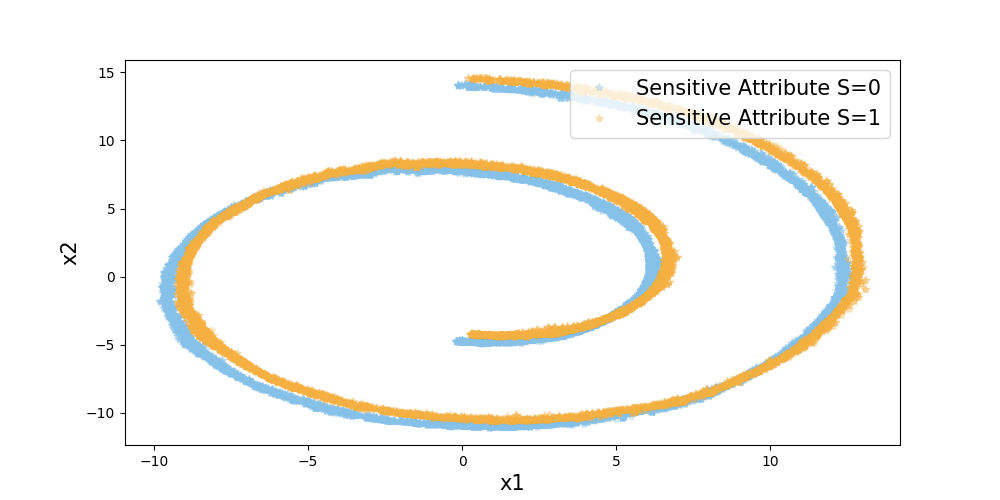}
    \caption{No fairness constraint}
    \label{fig:sw0}
    \end{subfigure}
     \begin{subfigure}{0.24\textwidth}
    \includegraphics[width=\linewidth]{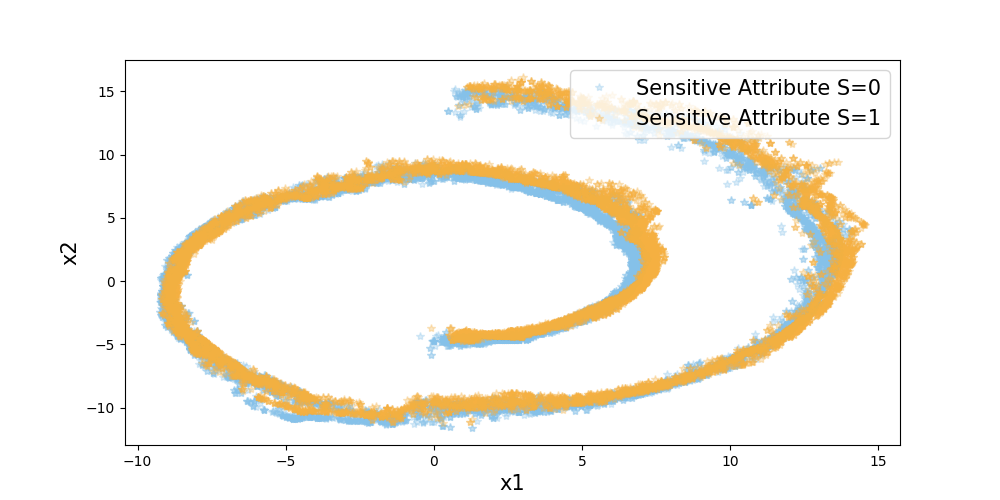}
   \caption{AdvCE (\cite{edwards2015censoring})}
    \label{fig:sw1}
    \end{subfigure}
     \begin{subfigure}{0.24\textwidth}
    \includegraphics[width=\linewidth]{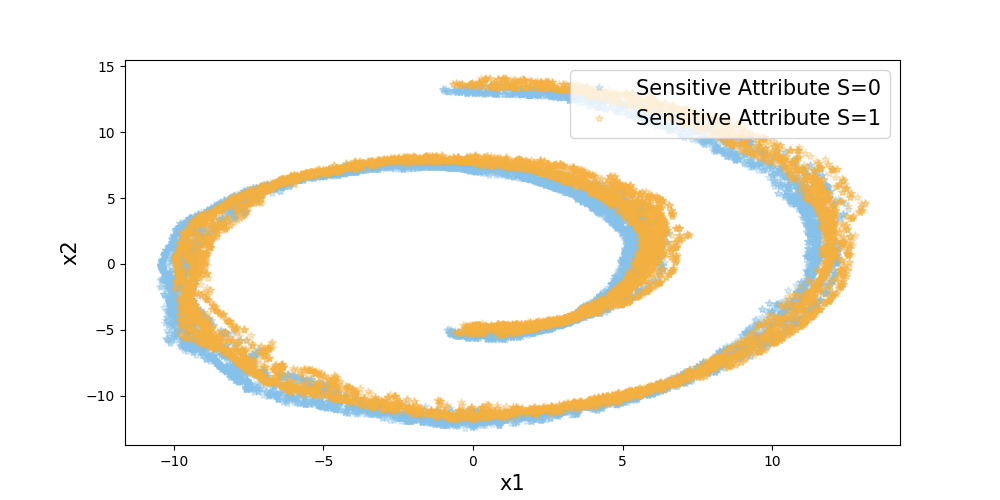}
   \caption{AdvL1 (\cite{madras2018learning})}
    \label{fig:sw2}
    \end{subfigure}
     \begin{subfigure}{0.24\textwidth}
    \includegraphics[width=\linewidth]{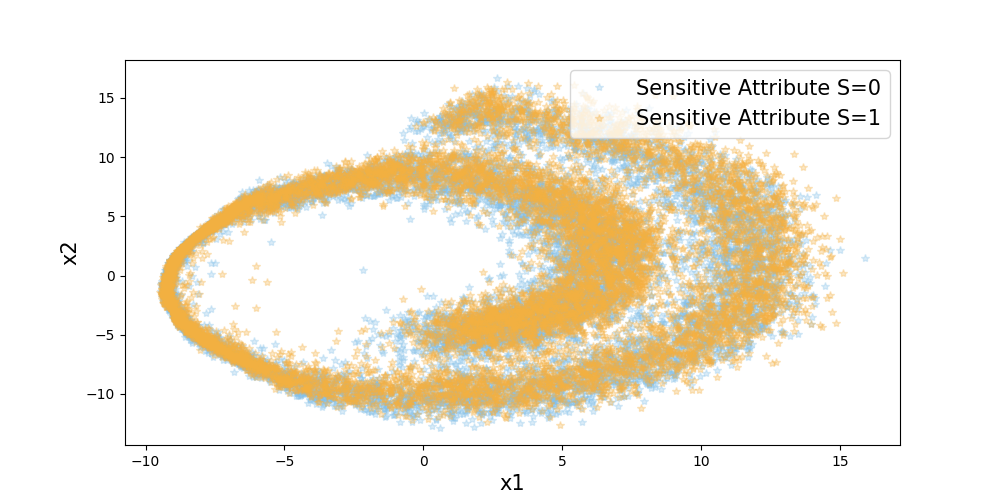}
    \caption{AWGN (our work)}
    \label{fig:sw3}
    \end{subfigure}
    \caption{\textbf{Without additional guarantees, fair representation learning could still leak information related to sensitive attribute.} This visualizes in a 2D-plane representations from the Swiss Roll data set once passed through a decoder. The easier it is to visually distinguish sensitive attributes, the more information related to sensitive attribute the representation leaks.}
    \label{fig: sw}
    \vskip -0.2in
\end{figure*}

In practice, the representation mapping $t$ and the decoder $g$ are modelled by neural networks. An AWGN channel is added to $t$ to learn a smoothed representation distribution $\mu_{t*\sigma}$. The data controller trades off minimizing a reconstruction loss $\mathcal{L}_{rec}(t, g)=E_{x}[l_{rec}(t, g, x)$ with minimizing demographic unparity $\mathcal{L}_{DP}(t)= \Delta^{*}(t_{\sigma})$. With a sample $\mathcal{D}_{n}=\{(x_{i}, s_{i})\}_{i=1}^{n}$, the data controller uses the plug-in auditor and solves the empirical minimization problem as
\begin{equation}
\label{eq: obj_sample}
    \min_{t,g}\frac{1}{n}\displaystyle\sum l_{rec}(t, g, x_{i}) + \lambda \Delta(f_{n}^{plug}, t_{\sigma}),
\end{equation}
where $\lambda$ controls for the strength of the fairness constraint imposed on the representation distribution. The minimization problem in \eqref{eq: obj_sample} differs from previous work on fair representation learning because of the noise added to $Z$.  The main advantage of convolving the representation distribution with a Gaussian noise is that the finite-sample fairness constraint $\Delta(f_{n}^{plug}, t_{\sigma}, )$ approximates $\Delta^{*}(t_{ \sigma})$ at a rate $O(n^{-1/2})$, while previous work does not offer this guarantee (see section \ref{sec: 2}).

The second advantage is that the empirical demographic parity certificate can be computed without modelling the auditor by an additional neural network. This is because we can use our empirical estimates \eqref{eq: mu_n} of the class-conditional densities to estimate the posterior distribution $\eta(z, s)=P(S=s|Z=z)$ as $\eta_{n}(z, s) = \mu_{n, \sigma}(z| S=s)/\mu_{n, \sigma}(z)$, where $\mu_{n, \sigma}(z) = \mu_{n, \sigma}(z, 1) + \mu_{n, \sigma}(z, 0)$. Since $\Delta^{*}(t)$ relates to the balanced error rate of predicting the sensitive attributes (see proof of \ref{th: 2}
 or \cite{feldman2015certifying}), we can write $\Delta^{*}(t)=\mathcal{L}_{DP}(\mu_{t, \sigma})$, where $\mathcal{L}_{DP}(\mu_{t, \sigma})=E_{z\sim \mu_{t, \sigma}}[|\eta(z, 1) -\eta(z, 0)|]$ (see \cite{zhao2013beyond}). Our approach relies on two results: (i) for any finite sample of size $n$, $\mathcal{L}_{DP}(\mu_{n, \sigma})$ approximates well $\mathcal{L}_{DP}(\mu_{t*\sigma})$; (ii) $\mathcal{L}_{DP}(\mu_{n, \sigma})$ can be estimated efficiently by Monte-Carlo estimation. The first observation uses the following result, which is a consequence of Theorem \ref{th: 3}
 
 \begin{thm}
 \label{th: 4}
 Let $\sigma > 0$ and $n\geq 1$. For all representation mapping $t:\mathcal{X}\rightarrow \mathcal{Z}$
 \begin{equation}
 \begin{split}
  \sup_{\mu}E_{\mathcal{D}_{n}}|\mathcal{L}_{DP}(\mu_{t* \sigma}) - \mathcal{L}_{DP}(\mu_{n, \sigma})| & \\ \leq 2\exp\left(\frac{||t||_{\infty}^{2}}{2\sigma^{2}}\right) (n_{0}^{-1/2} + n_{1}^{-1/2}). & \\
  \end{split}
 \end{equation}
 \end{thm}
 Therefore, we can use $\mathcal{L}_{DP}(\mu_{n, \sigma})$ as an approximation of $\mathcal{L}_{DP}(\mu_{t*\sigma})$. That is,  in place of $\mu_{t, \sigma}$, we propose to use the distribution $\mu_{n, \sigma}$, for which $\eta_{n}$ is the posteriori probability. Moreover,  $\mathcal{L}_{DP}(\mu_{n, \sigma})$ can be efficiently approximated by Monte Carlo integration. For a sample of features $\mathcal{D}_{n}=\{(x_{i}, y_{i})\}_{i=1}^{n}$, $\mu_{n, \sigma}^{0}$ and $\mu_{n, \sigma}^{1}$ are mixtures of $d$-dimensional Gaussians. Thereby, we approximate $\mathcal{L}_{DP}(\mu_{n, \sigma})$ with 
 \begin{equation}
    \hat{ \mathcal{L}}_{DP}(\mu_{n, \sigma})=\frac{1}{nm}\displaystyle\sum_{i=1}^{n}\sum_{j=1}^{m}E_{\epsilon}[|\eta_{n}(z_{ij}, 1) -\eta_{n}(z_{ij}, 0)|,
 \end{equation}
 where $z_{ij}=t(x_{i})+noise_{ij}$, $\{noise_{ji}\}$ is a vector of $n\times m$ draws from a d-dimensional Gaussian $\mathcal{N}(0, \sigma I_{d})$ and $m$ is the number of draws per sample point. $\hat{ \mathcal{L}}_{DP}(\mu_{n, \sigma})$ is an unbiased approximation of $\mathcal{L}_{DP}(\mu_{t, \sigma})$ and achieves a Mean-Squared-Error (MSE) of order $O(n^{-1}m^{-1})$ (see proof of Theorem 4 in appendix).

To sum up, the data controller learns $(t, g)$ by minimizing the following combined empirical loss
 \begin{equation}
\label{eq: obj_sample_final}
    \min_{\theta, \varphi}\frac{1}{n}\displaystyle\sum_{i} l_{rec}(t, g, x_{i}) + \lambda \hat{ \mathcal{L}}_{DP}(\mu_{n, \sigma}).
\end{equation}
 
 %This loss differs from previous fair representation learning (\cite{madras2018learning} or \cite{edwards2015censoring}) since (i) the fairness component of the loss $\hat{ \mathcal{L}}_{DP}(\mu_{n, \sigma})$ does not require training an additional adversarial neural network to be estimates; and (ii), $\hat{ \mathcal{L}}_{DP}(\mu_{n, \sigma})$ is a good approximation of an upper bound for the demographic parity of any downstream data processors that will use sample from the representation distributions.
\textbf{Practical implementation.} We minimize the loss \eqref{eq: obj_sample_final} by stochastic gradient descent. Each mini-batch is split in half: the first half is used to estimate $\mu_{n, \sigma}$ as in \eqref{eq: mu_n}; the second half to estimate the loss in \eqref{eq: obj_sample_final} with $m=1$. At the end of training, we compute a leave-one-out balanced error rate $BER(f^{plug}_{n})$ for the plug-in auditor on both a test and train samples and infer an empirical certificate as $\Delta(f_{n}^{plug}, t)= 1 - 2BER(f^{plug}_{n})$ (see \cite{feldman2015certifying}). The Gaussian noise $\sigma$ is an hyper-parameter chosen so that empirical certificate estimated on train and test data are similar. 

\section{Experiments}

\begin{figure*}
\vskip 0.2in
 \centering
     \begin{subfigure}{0.45\textwidth}
    \includegraphics[width=\linewidth]{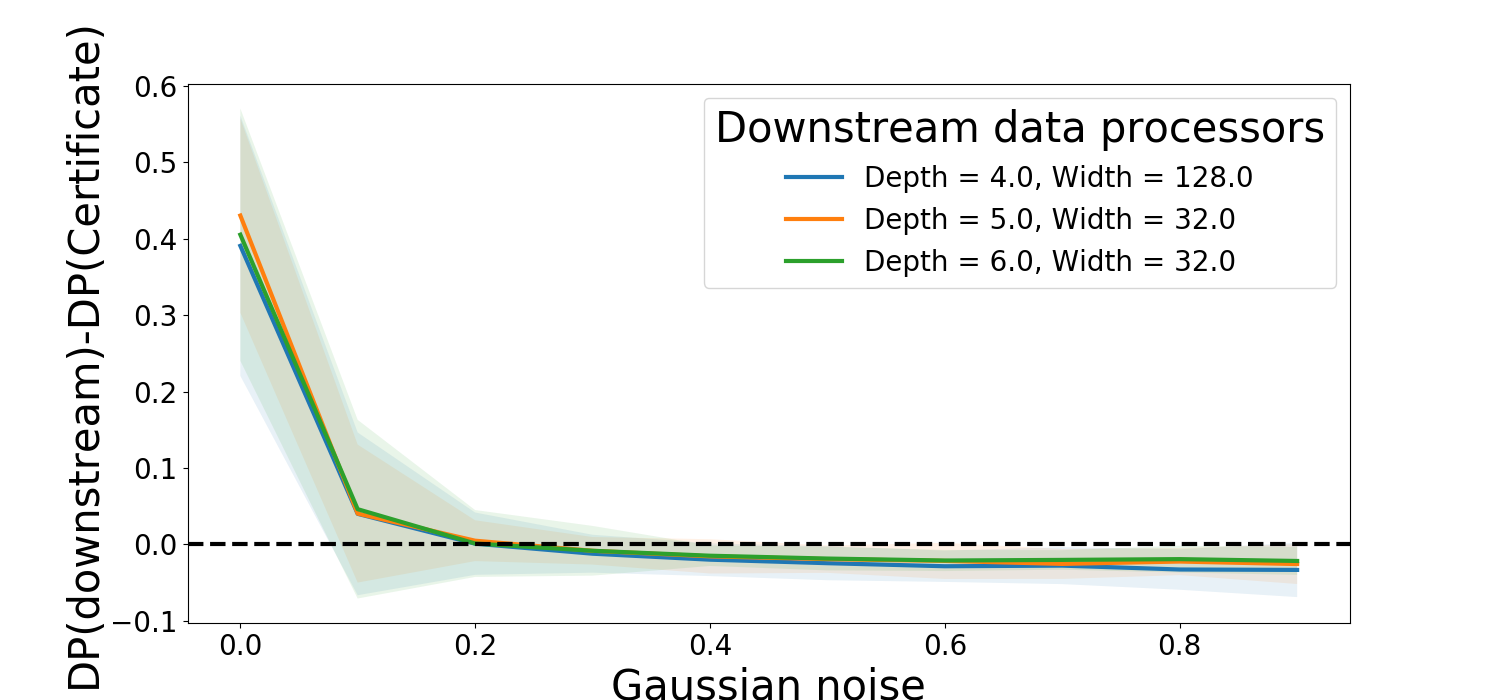}
    \caption{Swiss Roll dataset}
    \label{fig:robustness_error_sw}
    \end{subfigure}
     \begin{subfigure}{0.45\textwidth}
    \includegraphics[width=\linewidth]{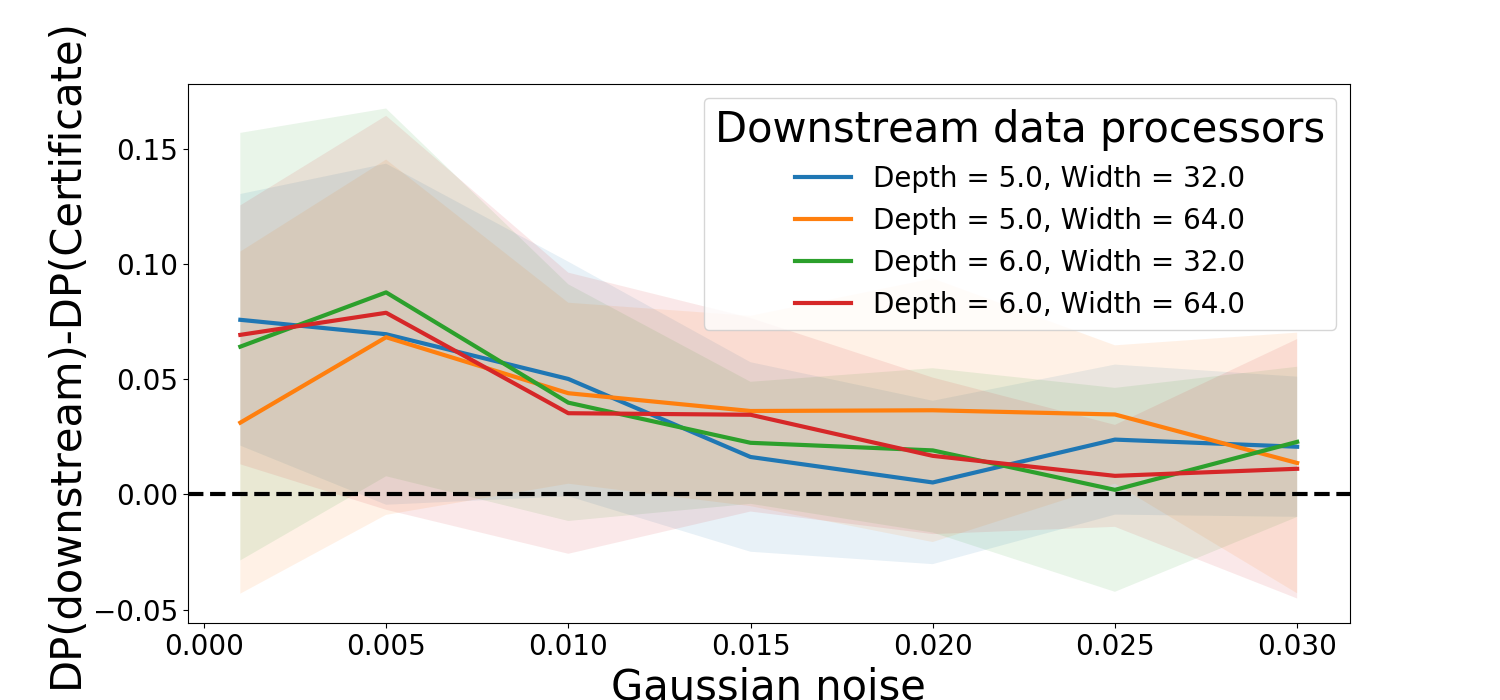}
    \caption{DSprites dataset}
    \label{fig:robustness_error_dsprites}
    \end{subfigure}
    \caption{
    %\textbf{Smoothing with an additive Gaussian white noise (AWGN) guarantees empirical demographic parity certificates are good approximation of the demographic parity for the underlying representation distribution.} 
    \textbf{Generalization properties of empirical demographic parity certificates obtained by adding an additive Gaussian white noise (AWGN) channel to fair representation learning.} This shows the difference between the demographic parity measured by the empirical certificate and the one obtained by downstream processors that attempt to predict the sensitive attribute while observing fresh samples from the representation distribution. Differences below the horizontal dashed line indicate that the empirical certificate is a reliable approximation of the demographic parity of the representation distribution. Shaded areas captures the one standard deviation  around the median of $100$ simulations. }
    \label{fig: effect_noise}
    \vskip -0.2in
\end{figure*}

\subsection{Synthetic Datasets}
Our first synthetic data consists of two $3D$ Swiss rolls: one for $S=0$ and one shifted South-West for $S=1$ (see \ref{fig:sw0}). We use $20,000$ samples for training the autoencoder $(t, g)$ and $10,000$ fresh samples to train the downstream test functions. The autoencoder is a neural network with seven hidden layers and $32$ neurons each with RELU activation and is trained with a learning rate of $0.001$ for 400 epochs. The value of the fairness coefficient $\lambda$ in \eqref{eq: obj_sample_final} is $5$. The representation space has dimension $3$ ($d=3$).

Our second synthetic data is a variant of the DSprites dataset (\cite{dsprites17}) that contains $64$ by $64$ black and white images of various shapes (heart, square, circle). 
%and is used to benchmarked methods that learn disentangled representations of a distribution. 
Since fair representation mapping consists of disentangling sensitive attributes from the rest of the features, DSprites offers an interesting challenge (\cite{Creager2019FlexiblyFR}). The DSprites dataset has six independent factors of variation: color (black or white); shape (square, heart, ellipse), scales (6 values), orientation (40 angles in $[0, 2\pi]$); x- and y- positions (32 values each). We adapt the sampling to generate a source of potential unfairness. We consider shape as the sensitive attribute. Following \cite{Creager2019FlexiblyFR}, we assign to each possible combination of attributes a weight proportional to $\frac{i_{shape}}{3} + \left(\frac{i_{X}}{32}\right)^{3}$ , where $i_{shape}\in \{0, 1, 2\}$ and $i_{X}=\{0, 1, ..., 21\}$. Then, we sample $60,000$ combinations of the six factors of variations according to the weights. We use $50,000$ samples to train the autoencoder and $10,000$ to train the downstream test functions. The autoencoder architecture -- borrowed from \cite{Creager2019FlexiblyFR} -- includes $4$ convolutional layers and $4$ deconvolutional layers and uses RELU activation. The model is trained with a learned rate of $0.0001$ for $400$ epochs. The value of the fairness coefficient $\lambda$ in \eqref{eq: obj_sample_final} is $0.2$. The dimension of the representation space is $10$.  

\textbf{Effect of noise on certificate reliability.} For both datasets, we first learn an encoder-decoder mapping $(t, g)$ with an increasing amount of Gaussian noise; estimate an empirical $\Delta(f_{n}^{plug}, t)-$demographic parity certificate; and then, test whether $\Delta(f_{n}^{plug}, t)$ is larger than the demographic parity $\Delta(f, t)$ of different downstream test functions $f$. Our test functions predict sensitive attributes from new samples of the representation distribution. We model them as fully connected neural networks with $4$ to $6$ hidden layers with $32$ to $128$ neurons each. Each test function is trained for $400$ epochs with a learning rate of $0.001$. After the autoencoder is trained, its weights are frozen, and fresh representations are generated by $10,000$ forward passes of the encoder on the test data. The generated fresh representations form the inputs of the test functions. 

%We construct them by stacking the layers of the decoder $g$ with $4$ to $6$ hidden neural layers of a fully connected network. The weights on the decoder part of the test functions are frozen, while the weights on the fully connected part of the test functions are learned to minimize a binary cross-entropy loss, using sensitive attributes as labels. 

Figure \ref{fig: effect_noise} shows that the AWGN channel improves how empirical certificates approximate the demographic parity of the representation distribution. As the Gaussian noise $\sigma$ increases, the difference between the demographic parity of downstream test functions and of the empirical certificate decreases. For the Swiss Roll dataset (see Figure \ref{fig:robustness_error_sw}), with $\sigma^{2}>0.1$, the $\Delta(f_{n}^{plug}, t)$ empirical certificate upper bounds the demographic parity of any of the downstream test functions we built, regardless of their complexity. For the DSprites dataset, the empirical certificate approximates better the demographic parity obtained by the downstream test functions for $\sigma^{2}>0.02$ (see Figure \ref{fig:robustness_error_dsprites}). Moreover, the variance of $\Delta(f_{n}, t) - \Delta(f, t)$ decreases as the Gaussian noise increases. This is consistent with the upper bound in Theorem \ref{th: 3}, which decreases with smaller signal-to-noise ratio $||t||_{\infty}/\sigma$.

\textbf{Comparative adversarial approaches.} We benchmark the use of an AWGN channel with comparative approaches in fair representation learning based on an adversarial auditor ($AUD)$ trained with (i) a cross-entropy loss (AdvCE, \cite{edwards2015censoring});  or, with (ii)a group L1 loss (AdvL1, \cite{madras2018learning}). 

AdvCE is a fair representation learning method from \cite{edwards2015censoring}. The auditor is modeled as an adversarial neural network $f$ that predicts sensitive attributes from samples of the representation distribution and minimizes the following cross-entropy loss:
\begin{equation}
\label{eq: ce}
 \mathcal{L}_{CE}(f)=-\frac{1}{n}\displaystyle\sum_{i=1}^{n}s_{i}\log(f(x_{i}) + (1-s_{i})\log(1-f(x_{i})).
\end{equation}
Moreover, the autoencoder is trained to minimize a loss $\mathcal{L}_{rec} -\lambda\mathcal{L}_{CE}(f)$. 

AdvL1 (\cite{madras2018learning}) replaces the cross-entropy loss by a group L1 loss: instead of \eqref{eq: ce}, the adversary minimizes
\begin{equation}
    \mathcal{L}_{L1} = \frac{1}{n_{0}}\displaystyle\sum_{i, s_{i}=0}f(x_{i}) - \frac{1}{n_{1}}\displaystyle\sum_{i, s_{i}=1}f(x_{i}),
\end{equation}
and the autoencoder minimizes $\mathcal{L}_{rec} -\lambda\mathcal{L}_{L1}(f)$.

For both AdvCE and AdvL1, the autoencoder is the same as in the experiments in Figure \ref{fig: effect_noise}. The adversarial auditor is modeled as a neural network with seven hidden layers of $32$ neurons each. For both Swiss Roll and DSprites, the autoencoder is trained for $400$ epochs with a learning rate of $0.0001$; the adversary with a learning rate of $0.001$. For Swiss Roll, the fairness coefficient $\lambda$ is $20$ for both AdvCE and AdvL1 and $5$ for AWGN; for DSprites, $\lambda$ is $0.1$ for both AdvCE and AdvL1 and $0.15$ for AWGN. The downstream processor ($PROC$) stacks the (frozen) layers of the decoder $g$ with $4$ hidden neural layers of a fully connected network. Crucially for our experiment, the decoder does not use any information related to sensitive attribute.
%Weights for the decoder part of the downstream processor are frozen to the one obtained after training the autoencoder. 

Table \ref{tab: 1} measures the performance of empirical certificates as the difference $PROC-AUD$ between the unparity measured by $PROC$ and $AUD$: the lower the difference, the more reliable is the empirical certificate. For both comparative methods, the adversarial auditor $AUD$ estimates the representation distribution to be almost independent of sensitive attribute --  auditor's demographic parity is almost zero, but our downstream processor $PROC$ predicts with high accuracy the sensitive attribute and thus, has much higher demographic parity ($0.8$ for Swiss Roll and $0.4-0.5$ for DSprites). On the other hand, the introduction of an AWGN channel reduces the difference $PROC-AUD$ and guarantees that the empirical certificate estimated by the auditor $AUD$ approximates well the demographic parity of the downstream processor $PROC$. However, for the Swiss Roll dataset, a better approximation comes at the cost of a higher $L2-$ reconstruction loss.

\begin{table}[h]
\caption{\textbf{Robustness of fair representation learning}: This table compares for both Swiss Roll (SW) and DSprites (DS) datasets the use of an AWGN channel (\textbf{AWGN}) with alternative  fair representation learning methods that use adversarial auditors with cross-entropy loss (\textbf{AdvCE}, \cite{edwards2015censoring}) or group L1 loss (\textbf{AdvL1}, \cite{madras2018learning}). The lower the difference PROC-AUD between processor's (PROC) and auditor's (AUD) measure of unparity, the better empirical certificates approximate the representation's demographic parity. Results are the median of $100$ simulations for each method.}
\begin{center}
\begin{small}
\begin{sc}
\begin{tabular}{llllll}
\toprule
  Dataset & Model               & L2 & \multicolumn{3}{l}{Measured Unparity}\\
             &    & Loss & Aud & Proc & Proc-Aud \\
\midrule
SW &AdvCE & 0.09                  & 0.05       & 0.79 & 0.74        \\
SW &AdvL1 & 0.02                & 0.04       & 0.8  & 0.76       \\
SW &AWGN       & 1.19                  & 0.08       & 0.13  & \textbf{0.05}\\   
\midrule
DS & AdvCE & 0.01 & 0.02 & 0.54 & 0.52\\
DS & AdvL1 & 0.01 & 0.06 &0.39 & 0.33\\
DS & AWGN & 0.01 &0.14 &0.2 & \textbf{0.06}\\
\bottomrule
\end{tabular}
\label{tab: 1}
\end{sc}
\end{small}
\end{center}
\vskip -0.1in
\end{table}

Visually, Figure \ref{fig: sw} shows one of our simulation results for the Swiss Roll dataset and compares the representations generated by each comparative method and decoded by the downstream processor $PROC$. Although the three fair learning methods, AdvCE, AdvL1 and ours AWGN, certify the representation distribution to be almost independent of sensitive attribute, only the AWGN approach makes it really difficult to distinguish the sensitive attributes when looking at the decoded representation. 
%: it is consistent with the necessary condition in section \ref{sec: 2} that bounds the information flow from features to representations.  

\subsection{Application to Fair ML benchmark datasets}

\begin{figure}
     \begin{subfigure}{0.5\textwidth}
    \includegraphics[width=\linewidth]{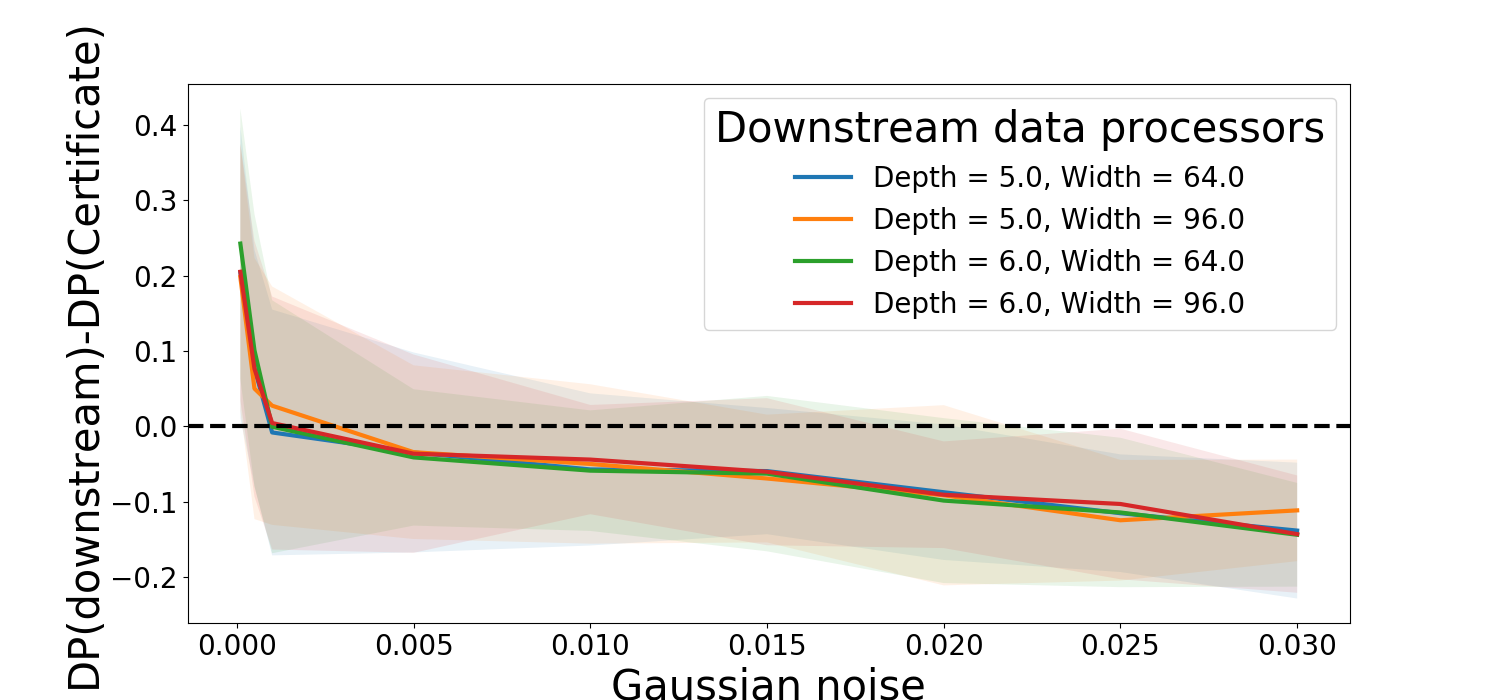}
    \caption{Adults dataset.}
    \label{fig: ad1}
    \end{subfigure}
     \begin{subfigure}{0.5\textwidth}
    \includegraphics[width=\linewidth]{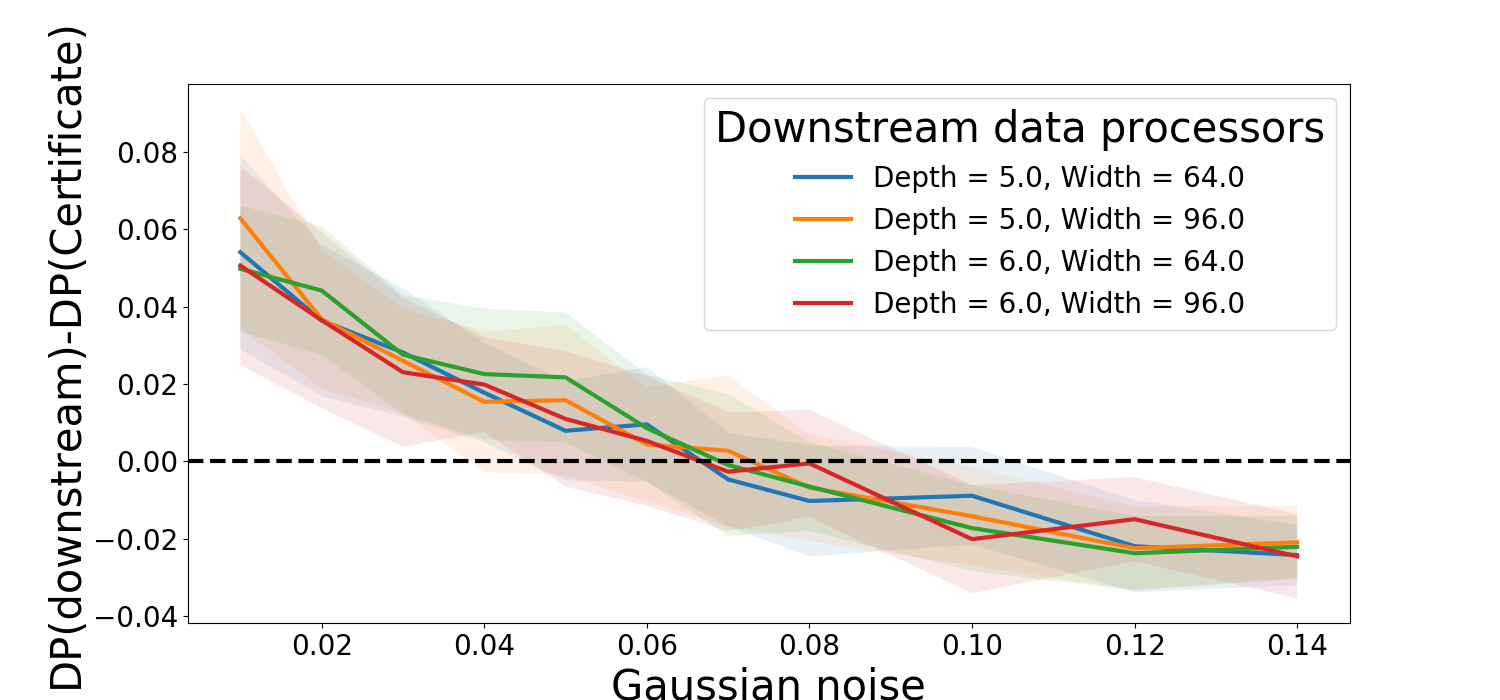}
    \caption{Heritage dataset.}
    \label{fig: ad2}
    \end{subfigure}
    \caption{\textbf{Generalization properties of empirical demographic parity certificates obtained by adding an additive Gaussian white noise (AWGN) channel to fair representation learning.} See  Figure \ref{fig: effect_noise}.}
    \label{fig: adh}
    \end{figure}

We apply our approach of fair representation learning with a AWGN channel to two fair learning benchmarks, Adults\footnote{https://archive.ics.uci.edu/ml/datasets/adult} and Heritage\footnote{https://foreverdata.org/1015/index.html}. The Adults dataset contains $49K$ individuals and includes information on $10$ features related to professional occupation, education attainment, race, capital gains, hours worked and marital status. The sensitive attribute is the gender to which individuals self-identify to. The data is split into a $34K$ train set and a $15K$ test set.

The Health Heritage dataset contains $220K$ individuals with $66$ features related to age, clinical diagnoses and procedure, lab results, drug prescriptions and claims payment aggregated over $3$ years. The sensitive attribute is the gender to which individuals self-identify to. After removing individuals with missing records, we split the data into a $142K/35K$ train/test split.

For both Adults and Heritage the autoencoder has seven hidden layers of $32$ neurons each and is trained for $400$ epochs with a learning rate of $0.0005$. The dimension of the representation latent space is $10$ for Adults and $24$ for Heritage. The fairness coefficient $\lambda$ is $3$ for Adults and $0.5$ for Heritage. Downstream test functions are trained as in Figure \ref{fig: effect_noise}. 

%The Adults data contains $49K$ individuals and includes information related to income, professional occupation, education attainment and marital status. The Heritage data contains $220K$ individuals with $66$ features related to clinical diagnoses, lab results, drug prescription and claims payment. In both, the sensitive attribute is whether an individual is self-identified as male or female. The representation mapping is learned on a training sample of $37K$ individuals for Adults and $142K$ individuals for Heritage. 

\textbf{Effect of noise on certificate reliability.} In Figure \ref{fig: adh}, downstream data processors probe the demographic parity of the learned representation distribution by predicting the gender of representations encoded from a test sample of $15K$ individuals for Adults and $35K$ individuals for Heritage. As in Figure \ref{fig: effect_noise}, we compare the demographic parity of these downstream processors to the empirical certificate estimated after training.
%Results on the Adults dataset are consistent with what we obtained on synthetic data. Figure \ref{fig: ad1} compares the demographic parity certificate stamped by the plug-in estimator to the one generated by downstream neural networks of various depth and width. 
For both datasets, Figure \ref{fig: adh} confirms that (i) without the noisy channel (small $\sigma$), the empirical certificate underestimates the demographic unparity of downstream data processors; but, (ii) that the AWGN channel is sufficient for empirical certificates to upper-bound the demographic unparity obtained by various downstream users.

\begin{figure}
     %\begin{subfigure}{0.5\textwidth}
    \includegraphics[width=\linewidth]{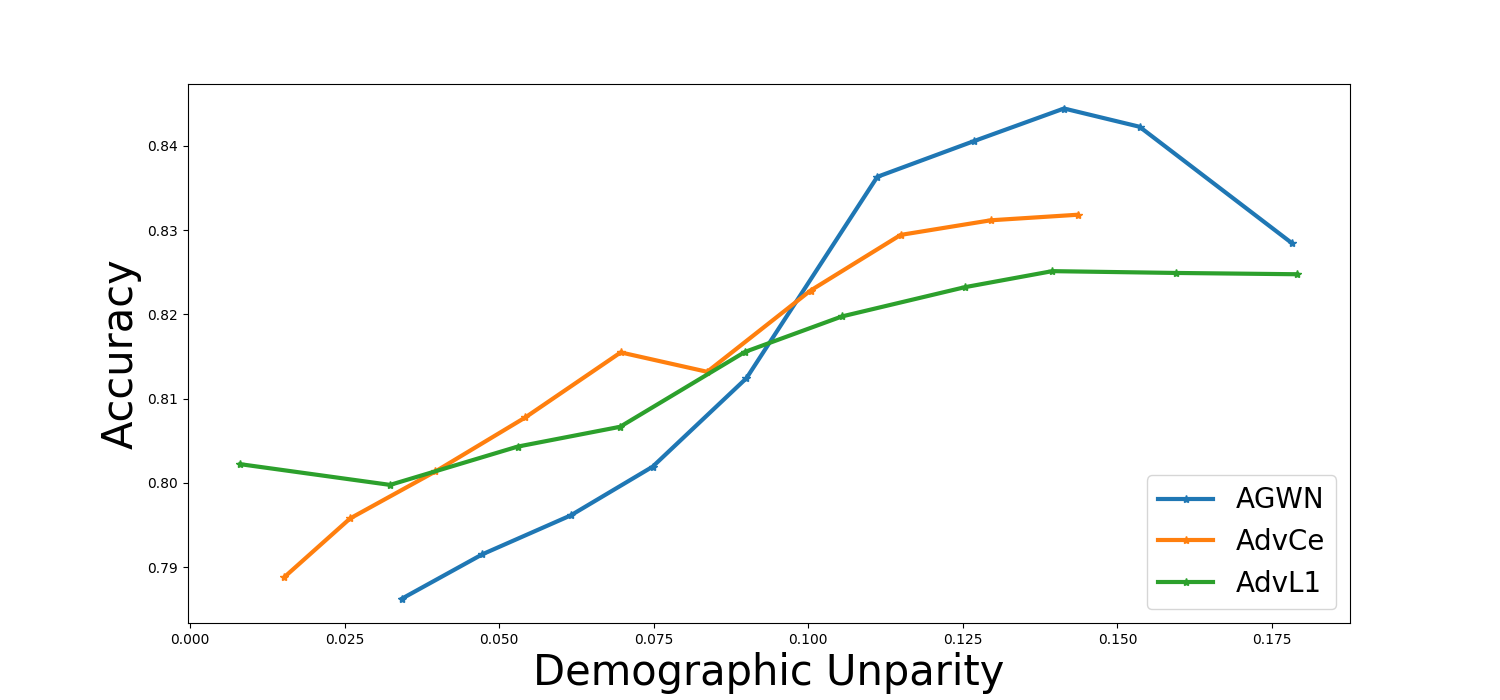}
    %\caption{Adults dataset}
    %\label{fig: pf1}
    %\end{subfigure}
    %\begin{subfigure}{0.5\textwidth}
   % \includegraphics[width=\linewidth]{figures/pareto_front_adults_median.png}
   % \caption{Placeholder}
   % \label{fig: pf2}
   % \end{subfigure}
    \caption{\textbf{Pareto Front for fair representation learning approaches.} This shows an accuracy-fairness trade-off by measuring the accuracy and the demographic parity of neural networks of various depth and width that detect wealthy individuals from representations of the Adult dataset.}
    \label{fig: pf}
    \end{figure}

\textbf{Accuracy-fairness trade-off.} To explore how the AWGN channel affects the information contained in the representation distribution, we retrain on the Adults dataset the three fair learning methods -- AdvCE, AdvL1 and AWGN -- but leave out the feature related to income. We generate from the test samples their corresponding representations and predict whether their income is over $50K$. 

We construct Pareto fronts (see \cite{madras2018learning}) by sweeping the parameter space for values of the fairness coefficient $\lambda$ between $0$ and $4$ and by training for each value of $\lambda$ an autoencoder for the three comparative methods, AdvCE, AdvL1 and AWGN. The AWGN channel uses a noise $\sigma$ of $0.015$. Both adversarial methods use auditors modeled as neural networks with six hidden layers that are trained with a learning rate of $0.005$. After we train the autoencoder, we freeze its weights, generate fresh samples from the representation distribution using the test data of $15K$ individuals, and predict whether an individual earns more than $50K$ a year. We apply this downstream task to four classifiers that we model as neural networks of depth varying from 3 to 6. We repeat the simulation $10$ times and thus obtain $40$ demographic parity/accuracy tuples for each value of $\lambda$ and each of the three fair representation learning methods (AdvCE, AdvL1 and AWGN). We bin the demographic parity values in $10$ buckets and report $75\%-$ quantile accuracy attained within each bin.

%We sweep the parameter space for values of the fairness constraint $\lambda$ in \eqref{eq: obj_sample} between $0$ and $4$, bin the resulting values of demographic parity and compute the $75\%-$ quantile accuracy attained within each bin. 

The resulting Pareto fronts in Figure \ref{fig: pf} show that compared to alternative fair learning methods, the AWGN channel does not appear to degrade significantly the accuracy-fairness trade-off: the trade-off is worse at low level of demographic parity and better at higher level of demographic parity. Although the AWGN channel limits the maximum amount of information that is transferred from the data to the representation (see \cite{cover2012elements}), it also allows for a better empirical approximation of demographic parity and thus helps guiding the representation mapping toward the correct fairness-information trade-off. 

\section{Conclusion}
This paper investigates whether a data controller could generate representations of the data with fairness guarantees that would hold for any downstream processor using samples from the representation distribution. We show that for demographic parity certificate to approximate well the demographic parity of all future data processors it is necessary and sufficient to bound the $\chi^{2}$ mutual information between feature and representation. To meet this condition, we show the benefit of adding an AWGN channel while learning a fair representation of the data.

Our work opens promising research avenues in fair representation learning.  An AWGN channel may be only one of many approaches to bound the $\chi^{2}$ mutual information between feature and representation. A comparison of competitive approaches would be crucial to improve the accuracy-fairness trade-off of learning reliably fair representations. 

\section{Appendix - Proofs of Results}

\subsection{Proof of Theorem \ref{th: 1}}
The proof of Theorem \ref{th: 1} uses the following lemma (from \cite{feldman2015certifying}) that links the demographic parity of a test function $f$ and its balanced error rate $BER(f)$,
\begin{equation}
    BER(f, t) = \frac{P(f(Z)=1 | S=0) + P(f(Z)=0|S=1)}{2},
\end{equation}
where we make the dependence on the representation mapping $t$ explicit in $BER(f, t)$. 
\begin{lem}\cite{feldman2015certifying}
\label{lem: ber}
A representation space $(\mathcal{Z}, \mu_{t})$ satisfies an $\Delta^{*}(t)-$ demographic parity certificate if and only if 
\begin{equation}
\label{eq: ber}
   BER^{*}(t) \triangleq \min\limits_{f: \mathcal{Z}\rightarrow \{0, 1\}}BER(f, t) \geq \frac{1- \Delta}{2}. 
\end{equation}
Therefore, a representation space $(\mathcal{Z}, \mu_{t})$ can be stamped with a $\Delta^{*}(t)-$ demographic parity certificate with $\Delta^{*}(t) \equiv 1 -2 BER^{*}(t)$. 
\end{lem}

To prove the result in Theorem \ref{th: 1}, we consider a distribution $\mu$ over $\{1, 1/2, 1/3, ....\}$ and assume that $t$ is deterministic.  Therefore, since $t$ is deterministic, the representation distribution induced by $t$ is discrete and can be indexed by $k$. We denote $t(\mathcal{X})=\{z_{1}, z_{2}, ..., z_{K}\}$, with $K\leq \infty$ and $z_{k}\neq z_{k^{'}}$ for $k\neq k^{'}$. In this setting, the $\chi^{2}$ mutual information has an analytical form:

\begin{lem}
\label{lem: chi2}
Suppose that $t$ is a deterministic mapping from $\mathcal{X}$ to $\mathcal{Z}$. Assume that $t(\mathcal{X})=\{z_{1}, z_{2}, ..., z_{K}\}$ with $K\leq \infty$ and $z_{k}\neq z_{k^{'}}$ for $k\neq k^{'}$. Then, for all distribution over the features $\mathcal{X}$ such that $P(t(X)=z_{k}) > 0$, $I_{\chi^{2}}(X, Z)= K - 1$.
\end{lem}
\begin{proof}
First, since $t$ is a function, $P(Z=z_{k}|X=x)$ is equal to one if and only if $t(x)=z_{k}$. Therefore,
\begin{equation}\begin{split}
    I_{\chi^{2}}(X, Z) &= E_{x}\left(1-\frac{1}{P(Z=t(x)}\right)^{2}P(Z=t(x)) \\
    & = E_{x}\left[\frac{1}{P(Z=t(x)} \right]- 1 \\
    & = \displaystyle\sum_{k=1}^{K}\left[\frac{P(X, t(X)=z_{k})}{P(Z=z_{k})}\right] - 1 \\
    & = K - 1
    \end{split}
\end{equation}
\end{proof}
For each $k\in \{1, ..., K\}$, we choose one $x_{k}\in \mathcal{X}$ such that $t(x_{k})=z_{k}$. We parametrize a family of joint distributions $\mu(b)$ over $[0,1]\times \{0, 1\}$ as follows: $X$ is uniformly distributed over $\{x_{1}, ..., x_{K}\}$; and, for $b\in (0, 1)$, the sensitive attribute is given by $k^{th}$ binary expansion of $b$, where $X=x_{k}$. By Lemma \ref{lem: cor2}, the $\chi^{2}$ squared mutual information between $X$ and $t(X)$ is the same for any $b$ and equal to $K-1$. Moreover, since the sensitive attribute is a function of $t(X)$, $\Delta^{*}_{b}(t)=1$, where the subscript indicates that demographic parity is computed using the joint distribution $\mu(b)$ over $(Z, S)$. 

Let $B$ denote a random variable uniformly distributed on $[0,1]$. For any auditor $f_{n}$,
\begin{equation}
\label{eq: chi_div1}
    \begin{split}
        \sup\limits_{b\in [0,1]} E_{\mathcal{D}_{n}(b)}BER(f_{n}, t) & \overset{\mathrm{(a)}}{\geq} E_{B}E_{\mathcal{D}_{n}(B)}BER(f_{n}, t) \\
        & = E_{X, B}P[f_{n}((t(X), \mathcal{D}_{n}(B)) \neq  \\
        & S| t(X_{1}), ..., t(X_{n}), S_{1}, ... S_{n}, t(X)]\\
       & \overset{\mathrm{(b)}}{\geq}\frac{1}{2} P\left( \cap_{i=1}^{n}[ t(X) \neq t(X_{i})] \right)\\
       & \overset{(c)}{=}\frac{1}{2}\left(1 -\frac{1}{K}\right)^{n} \\
       & \overset{(d)}{=} \frac{1}{2}\left(1 -\frac{1}{I_{\chi^{2}}(X, Z)}\right)^{n}
    \end{split}
\end{equation}
where $(a)$ uses that the suppremum is larger than the average; (b) that for $Z\notin\{Z_{1}, ..., Z_{n}\}$, the sensitive attribute has a Bernouilli distribution with probability $1/2$; (c) that $X$ and then $Z$ is uniformly distributed; and, (d) that $I_{\chi^{2}}(X, Z) \leq K$ by Lemma \ref{lem: chi2}. Since $I_{\chi^{2}}(X, Z)$ is equal for all $b$, it follows from Lemma \ref{lem: ber} that 

\begin{equation}
    \sup_{b\in (0, 1)}\Delta^{*} -\Delta(f_{n}, t) \geq \left(1 - \frac{1}{I_{\chi^{2}}(Z, X)}\right)^{n}.
\end{equation}
Therefore, there exists a distribution $\mu$ over $\mathcal{X}\times \{0, 1\}$ such that for all auditors $f_{n}$, 

\begin{equation}
    E_{\mathcal{D}_{n}}|\Delta^{*} -\Delta(f_{n}, t)| \geq \left(1 - \frac{1}{I_{\chi^{2}}(Z, X)}\right)^{n}.
\end{equation}

\subsection{Proof of Corollary \ref{cor: rates}}
Suppose that $\inf\limits_{f_{n}\in \mathcal{F}_{n}}\sup\limits_{\mu}E_{\mathcal{D}_{n}}|\Delta^{*} -\Delta(f_{n}, t) |\leq \epsilon_{n}$ for some $\epsilon_{n} > 0$. Let $f_{n}\in \mathcal{F}_{n}$ be the auditor that reaches the minimum.

We have, for any distribution $\mu$ over $\mathcal{X}\times \{0, 1\}$,
\begin{equation}
\label{eq: cor1_eq}
\begin{split}
\left(1 - \frac{1}{I_{\chi^{2}}(Z, X)}\right)^{n} & \leq 
    \sup_{\mu} \left(1 - \frac{1}{I_{\chi^{2}}(Z, X)}\right)^{n}  \\
    & \overset{(a)}{\leq} \sup\limits_{\mu}E_{\mathcal{D}_{n}}|\Delta^{*} -\Delta(f_{n}, t) | \\
    &\leq \epsilon_{n},
    \end{split}
\end{equation}
where $(a)$ uses Theorem \ref{th: 1}. The result follows directly from equation \eqref{eq: cor1_eq}.

\subsection{Proof of Corollary \ref{cor: 1}}
We first show the following Lemma:

\begin{lem}
\label{lem: cor2}
Let $t$ be a function from $\mathcal{X}$ to $\mathcal{Z}$. Suppose that there exists a distribution $\mu$ over $\mathcal{X}\times \{0, 1\}$ such that $I_{\chi^{2}}(X, Z)=\infty$. Then, there exists an infinite countable set $\{a_{k}\}$ of $\mathcal{X}$ such that for all $k\neq k^{'}$, $t(a_{k})\neq t(a_{k^{'}})$.
\end{lem}

\begin{proof}
The proof proceeds by contradiction. Assume that we have a partition $t(\mathcal{X})=\{a_{k}\}_{k=1}^{M}$ for a fixed $M<\infty$. Given that $t$ is a function, as in the proof of Lemma \ref{lem: chi2}, 

\begin{equation}
\label{eq: inf_chi}
    \begin{split}
        \infty &= I_{\chi^{2}}(X, Z) = E_{x}\left[\frac{1}{P(Z=t(x)}\right] - 1 \\ &=\displaystyle\sum_{k=1}^{M}\frac{\mu(\{x|t(x)= a_{k}\})}{P(t(X)=t(a_{k}))} - 1 \\
        & \overset{(a)}{=} M,
    \end{split}
\end{equation}
where $(a)$ uses the fact that $\mu(\{x|t(x)= a_{k}\})=  P(t(X)=t(a_{k})) $. Equation \eqref{eq: inf_chi} contradicts $M < \infty$. 
\end{proof}

Therefore, by Lemma \ref{lem: cor2}, if for a distribution $\mu$ over $\mathcal{X}\times\{0, 1\}$, $I_{\chi^{2}}(Z, X)=\infty$, then there exists an infinite countable set $\{a_{k}\}$ of $\mathcal{X}$ such that $t$ takes a different value at each $a_{k}$. We choose $X$ to take value in $\{a_{k}\}_{k\geq 1}$ such that $P(a_{k})=p_{k}$ for $k\geq 0$ where the sequence $\{p_{k}\}_{k=1}^{\infty}$ will be chosen later on. As in the proof of Theorem \ref{th: 1}, we parametrize a family of distributions over $\mathcal{X}\times\{0, 1\}$ by $b\in (0, 1)$ such that for $X\in \{a_{1}, ...\}$, the sensitive attribute $S$ is the $k^{th}$ term of $b'$s binary expansion, where $X=a_{k}$. Because $S$ is a deterministic function of $X$, $\Delta^{*}(t)=1$. 

Let $B$ denote a random variable uniformly distributed on $[0,1]$. The rest of the proof follows the same steps as in the proof of Theorem \ref{th: 1}. For a sample point $X_{i}$, we denote $k_{i}$ such that $X_{i}=a_{k_{i}}$. For any auditor $f_{n}$,
\begin{equation}
\label{eq: chi_div_cor}
    \begin{split}
        \sup\limits_{b\in [0,1]} E_{\mathcal{D}_{n}(b)}BER(f_{n}, t) & \overset{\mathrm{(a)}}{\geq} E_{B}E_{\mathcal{D}_{n}(B)}BER(f_{n}, t) \\
        & = E_{X, B}P[f_{n}((t(X), \mathcal{D}_{n}(B)) \neq  \\
        & S| t(X_{1}), ..., t(X_{n}), S_{1}, ... S_{n}, t(X)]\\
       & \overset{\mathrm{(b)}}{\geq}\frac{1}{2} P\left( \cap_{i=1}^{n}[ k \neq k_{i}]\right)\\
       & \overset{\mathrm{(c)}}{=}\frac{1}{2}\displaystyle\sum_{k=1}^{\infty}p_{k}(1 - p_{k})^{n}\\
    \end{split}
\end{equation}
It remains to show that for all $\epsilon >0$, we can choose $\{p_{k}\}$ such that the right hand side of inequality \eqref{eq: chi_div_cor} is at least $1/2(1-\epsilon)$. Let $\epsilon >0$. We choose $p_{k}$ as follows. First, pick $K> \frac{1}{1-(1-\epsilon)^{1/n}}$.Then, let $p_{k}=1/K$ for $1 \leq k\leq K$ and $p_{k}=0$ elsewhere. It follows that
\begin{equation}
    \sup\limits_{b\in [0,1]} E_{\mathcal{D}_{n}(b)}BER(f_{n}, t) \geq \frac{1}{2}\left(1-\frac{1}{K}\right)^{n}\geq \frac{1}{2}(1-\epsilon).
\end{equation}

Therefore, using Lemma \ref{lem: ber}, we can conclude that for all $\epsilon >0$, there exists a distribution over $\mathcal{X}\times \{0, 1\}$ such that for all auditors $f_{n}$
\begin{equation}
    \Delta^{*}(t) - \Delta(f_{n}, t) \geq 1 -\epsilon. 
\end{equation}

\subsection{Examples of Representation Mappings without Finite Sample Guarantees}
\textbf{Injective mappings.} Suppose that $t$ is injective from $[0, 1]^{D}$ to $\mathbb{R}^{d}$. 

Consider $X$ distributed over the countable and infinite set $\{1, 1/2,... 1/k,....\}$ with $p_{k}= \kappa/k^{2}$ and $k^{-1}=\sum_{k=1}^{\infty}1/k^{2}$. By lemma \ref{lem: chi2}, $I_{\chi^{2}}(X, Z)=\infty$ and thus, by Corollary \ref{cor: 1}, there exists a distribution such that $\Delta^{*}(t) - \Delta(f_{n}, t)=1$ for all $f_{n}$. 

\textbf{Large $t(\mathcal{X})$.} Suppose that $|\{t(x)| x\in\mathcal{X}\}| \geq n/(\ln(n))^{\alpha}$, for some $\alpha < 1$.

By Lemma \ref{lem: chi2}, $I_{\chi^{2}}(X, Z)\geq n/(\ln(n))^{\alpha}-1$ and thus, by Corollary \ref{cor: rates}, if $\inf\limits_{f_{n}\in \mathcal{F}_{n}}\sup\limits_{\mu} E_{\mathcal{D}_{n}}|\Delta^{*}(t) - \Delta(f_{n}, t) | = \epsilon_{n}$, then 
\begin{equation}
\begin{split}
    \frac{n}{(\ln(n))^{\alpha}} - 1 &\leq I_{\chi^{2}}(X, Z) \leq \frac{1}{1-\epsilon_{n}^{\frac{1}{n}}} \\
    & \overset{(a)}{\leq}\frac{n}{-\ln(\epsilon_{n})},
    \end{split}
\end{equation}
where $(a)$ uses that $e^{-x} \geq 1-x$. Therefore, $\epsilon_{n} \geq e^{-(\ln(n))^{\alpha}}=\omega(n^{-s})$ for $s>0$, since $\alpha < 1$. 

\subsection{Proof of Theorem \ref{th: 2}}
The proof Theorem \ref{th: 2} relies on a upper bound of $\Delta^{*}(t) - \Delta(f_{n}, t)$ that uses the total variation distance $TV(\mu_{t}^{s}, \mu_{n}^{s})$ between class conditional densities and their empirical counterpart:
\begin{equation}
    TV(\mu_{t}^{s}, \mu_{n}^{s}) = \displaystyle\int |\mu_{t}^{s} - \mu_{n}^{s})|dz.
\end{equation}

\begin{lem}
\label{lem: appendix 1}
Consider a sample $\{(z_{i}, s_{i})\}_{i=1}^{n}$ from a representation distribution $\mu_{t}$ induced by a representation rule $t$. Suppose that $\mu_{n}^{0}$ and $\mu_{n}^{1}$ are empirical density estimators of $P(Z|S=0)$ and $P(Z|S=1)$ respectively.  Denote $f_{n}$ the following auditing plug-in decision: for $z\in \mathcal{Z}$, $f_{n}(z) = 1$ if and only if $\mu_{n}^{1}(z) > \mu_{n}^{0}(z)$. Therefore, for all $n$
\begin{equation} \Delta(f_{n}, t) \leq \Delta^{*}(t) \leq \Delta(f_{n}, t) + 2\displaystyle\sum_{i=0, 1}TV(\mu_{t}^{i}, \mu_{n}^{i}).
\end{equation}
\end{lem}
\begin{proof}
Let $f^{*}$ denote the auditing rule that minimzes the balance error rate. Using \cite{devroye2013probabilistic} (ch 2), we show that for any auditing rule $f_{n}$
\begin{equation}
\label{eq: bn}
\begin{split}
2 - \displaystyle \int \eta_{f_{n}(z)}(z) \mu_{t}(dz)& =  2 - \displaystyle\sum_{i=0, 1} \int_{f_{n}(z)=i}\eta_{i}(z)\mu_{t}(dz)\\
& =  2 - \displaystyle\sum_{i=0, 1} \int_{f_{n}(z)=i}P(z|S=i)dz \\
& =2BER(f_{n}),
\end{split}
\end{equation}
where $\eta_{i}(z)$ is the balanced posteriori probability $\eta_{i}(z) = P(S=i|Z=z)/P(S=i)$. Moreover, 
\begin{equation}
\label{eq: bstar}
\begin{split}
2BER(f^{*}) & =2-  P(f^{*}(z)=1|S=1] \\ 
    & - P(f^{*}(z)=0|S=0)\\
 & = 2 - \displaystyle \int_{z, \mu_{t}^{1} > \mu_{t}^{0}} \mu_{t}^{1}(dz) - \displaystyle \int_{z, \mu_{t}^{0} > \mu_{t}^{1}} \mu_{t}^{0}(dz) \\
 & =2 - \displaystyle \int \max_{i}\eta_{i}(z) \mu_{t}(dz).
\end{split}
\end{equation}
Let denote $\eta_{n, i}$ the empirical estimate of $\eta_{i}$. Using equations \eqref{eq: bn} and \eqref{eq: bstar}, the proof of lemma \ref{lem: appendix 1} relies on the fact that

\begin{equation}
\label{eq: ber_tv}
\begin{split}
BER(f_{n})-BER(f^{*}) &= \displaystyle \int \max_{i}\eta_{i}(z) \mu_{t}(dz) \\ & - \displaystyle \int \eta_{f_{n}(z)}(z)  \mu_{t}(dz) \\
& = \displaystyle \int (\max_{i}\eta_{i}(z) - \max_{i}\eta_{n, i}(z)) \mu_{t}(dz)\\
& + \displaystyle \int (\eta_{n, f_{n}(z)}(z) - \eta_{f_{n}(z)}(z)) \mu_{t}(dz)\\
& \overset{(a)}{\leq} \displaystyle\sum_{i=0,1} \int |\eta_{i}(z) -\eta_{n, i}(z)| \mu_{t}(dz)\\
&= \displaystyle\sum_{i=0,1} \int |\mu^{i}_{t}(z) -\mu^{i}_{n}(z)|dz,
\end{split}
\end{equation}
 The inequality $(a)$ comes from the following observation. If the maxima are attained for the same $i\in \{0, 1\}$, then the right hand side integrand is equal to 0. Otherwise, suppose without loss of generality that $\max \eta_{i}(z)$ is reached for $i=0$, then the right hand side integrand is 

\begin{equation}
\begin{split}
\eta_{0}(z) - \eta_{n, 1}(z) + \eta_{n, 1}(z) -  \eta_{1}(z) &  =\eta_{0}(z) - \eta_{n, 0}(z) \\ & + \eta_{n, 1}(z) -  \eta_{1}(z) \\ &+ \eta_{n, 0}(z) - \eta_{n, 1}(z) \\
&\leq |\eta_{0}(z) - \eta_{n, 0}(z)| \\ &+ |\eta_{1}(z) - \eta_{n, 1}(z)|,
\end{split}
\end{equation}
where the inequality follows $\max_{i}\eta_{n, i}(z)=\eta_{n, 1}(z)$. The same argument can be applied when $\max \eta_{i}(z)=\eta_{1}(z)$. The result in lemma \ref{lem: appendix 1} follows from \eqref{eq: ber_tv}.
\end{proof}

The second part of the proof of theorem \ref{th: 2} is to show that the total variation distance between $\mu_{n}^{s}$ and $\mu_{t}^{s}$ is $O(1/\sqrt{n_{s}})$ for some empirical estimate of $\mu_{t}^{s}$:
\begin{lem}
\label{lem: appendix 2}
Consider a representation mapping $t:\mathcal{X}\rightarrow \mathcal{Z}$ and its induced distribution $\mu_{t}$. Assume that $I_{2}(Z, X) < \infty$. Then, for $s=0, 1$, define $\mu_{n}^{s}$ as 
\begin{equation}
\label{eq: mu_n}
    \mu_{n}^{s}(z)=\frac{1}{n_{s}}\displaystyle\sum_{i=1, s_{i}=s}^{n}P(z|X=x_{i})
\end{equation}
%Then, $\mu_{t}$ is absolutely continuous with respect to the Lebesgue measure $\nu$ and for $s=0, 1$
The total variation between $\mu^{s}_{t}$ and $\mu_{n}^{s}$ can be bounded as follows:
\begin{equation}\nonumber
   E_{\mathcal{D}\sim \mathcal{X}^{n}}\left[TV(\mu_{t}^{s}, \mu_{n}^{s})\right] \leq \sqrt{\frac{I_{2}(Z, X)}{n_{s}}}.
\end{equation}
\end{lem}

The upper bound of the total variation distance uses a Monte Carlo integration argument. For a sample $\mathcal{D}_{n}=\{x_{i}\}_{i=1}^{n}$, denote $\phi(z, x_{i})$ the probability $P(Z=z|X=x_{i})$. Therefore, $\mu_{t}(z) = E_{x\sim \mathcal{X}}[\phi(z, x)]$ and if $\mu_{n}^{s}$ is defined as in \eqref{eq: mu_n}, $\mu_{t}^{s}(z) = E_{{\bf X}, S=s}[\mu_{n}^{s}]$, where ${\bf X}=\{x_{i}\}_{i=1}^{n}\sim \mathcal{X}^{n}$.  Denote 
\begin{equation}
    \mathcal{E}^{s}({\bf X})=\displaystyle\int \left\vert\mu_{t}(z) -\frac{1}{n_{s}}\displaystyle\sum_{i=1, s_{i}=s}^{n}\phi(z, x_{i}) \right \vert dz,
\end{equation}
with $n_{s}=|\{i| s_{i}=s\}|$. We have
\begin{equation}
\label{eq: exp}
    \begin{split}
         E_{{\bf X}} [\mathcal{E}^{s}({\bf X})]  \overset{\mathrm{(a)}}{\leq}E_{{\bf X}}\left[\sqrt{\displaystyle\int \left(\frac{\mu_{t}(z) -\mu_{n}^{s}(z)}{\mu_{t}(z)} \right)^{2} \mu_{t}(z)dz}\right] &\\
         \overset{\mathrm{(b)}}{=} \frac{1}{n_{s}}E_{{\bf X}}\left[\sqrt{\displaystyle\int \displaystyle\sum_{i=1, s_{i}=s}^{n}\left(\frac{\mu_{t}(z) -\phi(z, x_{i})}{\mu_{t}(z)} \right)^{2} \mu_{t}(z)dz}\right] &\\
        \overset{\mathrm{(c)}}{\leq}\frac{1}{n_{s}}\sqrt{E_{{\bf X}}\left[\displaystyle\int \displaystyle\sum_{i=1, s_{i}=s}^{n}\left(\frac{\mu_{t}(z) -\phi(z, x_{i})}{\mu_{t}(z)} \right)^{2} \mu_{t}(z)dz\right]}&\\
         \overset{\mathrm{(d)}}{=} \frac{1}{n_{s}}\sqrt{ \displaystyle\sum_{i=1, s_{i}=s}^{n}E_{{\bf X}}\left[\displaystyle\int \left(\frac{\mu_{t}(z) -\phi(z, x_{i})}{\mu_{t}(z)} \right)^{2} \mu_{t}(z)dz\right]} &\\
        \overset{\mathrm{(e)}}{=} \sqrt{\frac{I_{2}(Z, X)}{n_{s}}}, &\\
    \end{split}
\end{equation}
where $(a)$ applies Cauchy-Schwarz inequality; $(b)$ uses the fact that the samples are independently drawn and that $E_{x_{i}}[\phi(z, x_{i})]=\mu_{t}(z)$; $(c)$ that  the squared-root is concave; $(d)$ that expectation and integral can be interchange; and, $(e)$ the definition of the chi-squared mutual information between $Z$ and $X$.

Putting lemma \ref{lem: appendix 1} and \ref{lem: appendix 2} together, we get the upper bound in theorem \ref{th: 2}.

\subsection{\texorpdfstring{$\chi^{2}$}{e} versus Classic Mutual Information}
Features are uniformly distributed over $[0, 1]$ and $t(x)=i$ for $x\in [1/(i+ 1), 1/i))$ and $i>0$. For each $i>0$, the sensitive attribute is constant over $[1/(i+ 1), 1/i))$ and equal to $1$ with probability $1/2$.

Form Lemma \ref{lem: chi2}, it is clear that $I_{\chi^{2}}(X, Z)=\infty$. On the other hand, we can show that the classic mutual information between $X$ and $Z$, $I_{Sh}(X, Z)$ is bounded. Since $t$ is deterministic,
\begin{equation}
\begin{split}
   I_{Sh}(X, Z) &=  \displaystyle\sum_{i=1}^{\infty}\frac{\ln(i(i+1))}{i(i+1)} \\
   & \leq \frac{\ln(2)}{2} +  \displaystyle\int_{1}^{\infty}\frac{\ln(x(x+1))}{x^2}dx\\
   & \overset{(a)}{=}  \frac{\ln(2)}{2} + 1 + \displaystyle\int_{1}^{\infty}\frac{1}{x(x+1)}dx \\
   &  \overset{(b)}{\leq }  \frac{\ln(2)}{2} + 2 < \infty, 
   \end{split}
\end{equation}
where $(a)$ and $(b)$ use integration by part and $(b)$ the fact that $1/x \geq 1/(x+1)$.

\subsection{Proof if Theorem \ref{th: 3}}
We only prove the upper bound on the $\chi^{2}$ mutual information since the remaining results in Theorem \ref{th: 3} follow directly from Theorem \ref{th: 2}.

Since the mapping $(p, q) \rightarrow q(p/q -1)^{2}$ is convex and since $Z$ is an infinite mixtures of Gaussians, we have that for $x\in\mathcal{X}$
\begin{equation}
\begin{split}
\label{eq: gauss_chi}
    \displaystyle\int \left(\frac{\mu_{t*\sigma}(z|X=x)}{\mu_{t*\sigma}(z)} - 1\right)^{2}\mu_{t*\sigma}(z) dz & \\
    \leq \displaystyle\int\int \left(\frac{\mu_{t*\sigma}(z|X=x)}{\mu_{t*\sigma}(z|X=x^{'})} - 1\right)^{2}\mu_{t*\sigma}(z|X=x^{'}) dz\mu(dx^{'})& \\,
    \overset{(a)}{=} \displaystyle\int \chi^{2}(z|X=x)|| z|X=x^{'}) \mu(dx^{'}), & \\
    \end{split}
\end{equation}
where we use Fubini Theorem to invert the summation over $z$ and $x^{'}$ and $(a)$ uses the definition of the $\chi^{2}$ divergence between $p(z|X=x)$ and  $p(z|X=x^{'}$. Since both $p(z|X=x)$ and  $p(z|X=x^{'}$ are Gaussians with variance $\sigma^{2}$ and mean $t(x)$ and $t(x^{'})$, respectively, the integrand in the right hand side of \eqref{eq: gauss_chi} can be computed analytically as 
\begin{equation}
\begin{split}
    \chi^{2}(z|X=x)|| z|X=x^{'}) = & \\ \frac{1}{2}\left[\exp\left(\frac{||t(x)-t(x^{'})||_{2}}{\sigma^{2}}\right) -1\right]. & \\
    \end{split}
\end{equation}
Therefore, 
\begin{equation}
\begin{split}
  I_{\chi^{2}}(X, Z) &\leq \frac{1}{2} E_{x, x^{'}} \left[\exp\left(\frac{||t(x)-t(x^{'})||_{2}^{2}}{\sigma^{2}}\right)\right] \\
  & \leq \frac{1}{2}\exp\left(\frac{2||t||_{\infty}^{2}}{\sigma^{2}}\right).
  \end{split}
\end{equation}

\subsection{Proof of Theorem \ref{th: 4}}
By \cite{zhao2013beyond}, we know that the balanced error rate of the optimal auditor $f^{*}$ is given by
\begin{equation}
\begin{split}
    BER(f^{*})&=\frac{1}{2} \displaystyle\int \min(\eta(z, 0), \eta(z, 1)) \mu_{t*\sigma}(dz) \\
    & =\frac{1}{4}\displaystyle\int(\eta(z, 0) + \eta(z, 1)) \mu_{t*\sigma}(dz)\\ 
    & -\frac{1}{4}\displaystyle\int|\eta(z, 0)- \eta(z, 1)| \mu_{t*\sigma}(dz) \\
    & \overset{(a)}{=} \frac{1}{2}-\frac{1}{4}\displaystyle\int|\eta(z, 0)- \eta(z, 1)| \mu_{t*\sigma}(dz),
    \end{split}
\end{equation}
where $(a)$ uses the definition of $\eta(z, s)= P(Z=z|S=s)/P(z)$. Therefore, by Lemma \ref{lem: ber}, 

\begin{equation}
    \mathcal{L}_{DP}(\mu_{t, \sigma})=\frac{1}{2}\displaystyle\int |\mu_{t, \sigma}^{0}(z) - \mu_{t, \sigma}^{1}(z)|dz
\end{equation}
and that
\begin{equation}
    \mathcal{L}_{DP}(\mu_{n, \sigma})=\frac{1}{2}\displaystyle\int |\mu_{n, \sigma}^{0}(z) - \mu_{n, \sigma}^{1}(z)|dz.
\end{equation}

Therefore, for any $t$ and any features distribution $\mu$ over the features $\mathcal{X}$,
\begin{equation}
    \begin{split}
|\mathcal{L}_{DP}(\mu_{n, \sigma})- \mathcal{L}_{DP}(\mu_{t, \sigma})|  \overset{\mathrm{(a)}}{\leq}\displaystyle\int |(\mu_{t, \sigma}^{0}(z) - \mu_{t, \sigma}^{1}(z)) &\\
        -(\mu_{n, \sigma}^{0}(z) - \mu_{n, \sigma}^{1}(z))|dz &\\
          \overset{\mathrm{(b)}}{\leq}\displaystyle\int |(\mu_{t, \sigma}^{0}(z) - \mu_{n, \sigma}^{0}(z))|dz & \\
         + \displaystyle\int |(\mu_{t, \sigma}^{1}(z) - \mu_{n, \sigma}^{1}(z))|dz &\\
         \overset{\mathrm{(c)}}{\leq} \exp\left(\frac{||t||_{\infty}^{2}}{\sigma^{2}}\right)\left(\sqrt{\frac{1}{n_{0}}} + \sqrt{\frac{1}{n_{1}}}\right), & \\
    \end{split}
\end{equation}
 where $(a)$ and $(b)$ are consequences of triangular inequalities; and $(c)$ follows from the definition of total variation distance, the upper bound in lemma \ref{lem: appendix 2} and theorem \ref{th: 3}.
 
 \subsection{Monte Carlo Approximation}
 \begin{lem}
 \label{lem: mc}
  let $m>0$ and $n>0$. Consider a sample $\{(x_{i}, s_{i})\}$ and a noise vector $\{noise_{ji}\}$ of $n\times m$ draws from a d-dimensional Gaussian $\mathcal{N}(0, \sigma I_{d})$. Denote $\mu_{n, \sigma}$ the empirical density as in \eqref{eq: mu_n} and for $i=1,..., n$ and $j=1, ..., m$ $z_{ij}=t(x_{i})+noise_{ij}$. If
 \begin{equation}
    \hat{ \mathcal{L}}_{DP}(\mu_{n, \sigma})=\frac{1}{nm}\displaystyle\sum_{i=1}^{n}\sum_{j=1}^{m}|\eta_{n}(z_{ij}, 1) -\eta_{n}(z_{ij}, 0)|
 \end{equation}
 then $\hat{ \mathcal{L}}_{DP}(\mu_{n, \sigma})$ is an unbiased estimator of $\mathcal{L}_{DP}(\mu_{n, \sigma})$ and 
 \begin{equation}
      E_{noise}\left[(\hat{ \mathcal{L}}_{DP}(\mu_{n, \sigma})- \mathcal{L}_{DP}(\mu_{n, \sigma}))^{2}\right] \leq  \frac{8||t||_{\infty}^{2} + 4\sigma^{2}}{\sigma^{2}}\frac{1}{nm}.
 \end{equation}
 \end{lem}
 
 \begin{proof}
 First, $\hat{ \mathcal{L}}_{DP}(\mu_{n, \sigma})$ is an unbiased estimator of $\mathcal{L}_{DP}(\mu_{n, \sigma})$ because
 \begin{equation}
 \begin{split}
     E_{noise}\left[\hat{ \mathcal{L}}_{DP}\right] & = \frac{1}{nm}\displaystyle\sum_{i=1}^{n}\sum_{j=1}^{m}E_{noise}[|\eta_{n}(z_{ij}, 1) -  \eta_{n}(z_{ij}, 0)|] \\
     & = \frac{1}{nm}\displaystyle\sum_{i=1}^{n}\sum_{j=1}^{m}\mathcal{L}_{DP}(\mu_{n, \sigma}) \\
     & = \mathcal{L}_{DP}(\mu_{n, \sigma}).
     \end{split}
 \end{equation}
 Therefore, the mean squared error can be written as
  \begin{equation}
  \begin{split}
      E_{noise}\left[(\hat{ \mathcal{L}}_{DP}(\mu_{n, \sigma})- \mathcal{L}_{DP}(\mu_{n, \sigma}))^{2}\right]  & \\ = \frac{1}{n^{2}m} \displaystyle\sum_{i=1}^{n}var_{noise}\left[k(x_{i} + noise)\right],&\\
     \end{split}
 \end{equation}
 where $k(z) = |\eta_{n}(z, 1) -  \eta_{n}(z, 0)|$. Moreover, by Gaussian Poincare inequality, 
 \begin{equation}
  \label{eq: GPI0}
 \begin{split}
     var_{noise}\left[k(x_{i} + noise)\right]  \overset{(a)}{\leq}  \sigma^{2}E_{noise}||\nabla k(x_{i}+noise)||_{2}^{2} & \\
      \overset{(b)}{=} 2\sigma^{2}\displaystyle\sum_{s}E_{noise}\left[||\nabla\log(\mu_{n, \sigma}^{s}(z, s))||_{2}^{2} \right] & \\
     \end{split}
 \end{equation}
 where $(a)$ uses the fact that the noise is Gaussian with standard deviation $\sigma$; $(b)$ that $z=x_{i}+noise$ and that $\nabla\eta_{n}(z, s) =  \eta_{n}(z, s)\nabla\log(\mu_{n, \sigma}^{s}(z, s)) + (1-\eta_{n}(z, s) \nabla\log(\mu_{n, \sigma}^{s}(z, 1-s))$. Moreover, for $s=0, 1$
 
 \begin{equation}
  \label{eq: GPI1}
     \begin{split}
     \nabla\log(\mu_{n, \sigma}^{s}(z, s)) & \overset{(a)}{=} \displaystyle\sum_{i=1}^{n}\nabla\log(\phi(z, x_{i})) P(X=x_{i}|z) \\
     & = -\frac{1}{2\sigma^{2}}\displaystyle\sum_{i=1}^{n}(z-t(x_{i}))P(X=x_{i}|z),
     \end{split}
 \end{equation}
 where $(a)$ denotes the Gaussian density with mean $t(x)$ and standard deviation $\sigma$ as $\phi(z, x)$. Therefore,
 
 \begin{equation}
  \label{eq: GPI2}
     ||\nabla\log(\mu_{n, \sigma}^{s}(z, s))||_{2} \leq \frac{||z||_{2} + ||t||_{\infty}}{\sigma^{2}}.
 \end{equation}
 \end{proof}
 Moreover, $z\sim \mu_{n, \sigma}$, which is a mixture of $n$ Gaussians, each with a non-central second moment equal to $\sigma^{2} + ||t(x_{i})||^{2}$. Therefore,
 \begin{equation}
 \label{eq: GPI3}
     E_{noise}||z||_{2}^{2} \leq \sigma^{2} + ||t||_{\infty}^{2}.
 \end{equation}
 By combining \eqref{eq: GPI0}, \eqref{eq: GPI1} \eqref{eq: GPI2} and \eqref{eq: GPI3}, we obtain that 
 \begin{equation}
      var_{noise}\left[k(x_{i} + noise)\right] \leq 4\frac{2||t||_{\infty}^{2} + \sigma^{2}}{\sigma^{2}},
 \end{equation}
 and thus that
 \begin{equation}
     E_{noise}\left[(\hat{ \mathcal{L}}_{DP}(\mu_{n, \sigma})- \mathcal{L}_{DP}(\mu_{n, \sigma}))^{2}\right] \leq 4\frac{2||t||_{\infty}^{2} + \sigma^{2}}{\sigma^{2}nm}
 \end{equation}

\bibliography{references}

\begin{thebibliography}{10}

\bibitem{ProPublica2016}
ProPublica.
\newblock How we analyzed the compas recidivism algorithm.
\newblock {\em ProPublica}, 2016.

\bibitem{pmlr-v81-buolamwini18a}
Joy Buolamwini and Timnit Gebru.
\newblock Gender shades: Intersectional accuracy disparities in commercial
  gender classification.
\newblock In Sorelle~A. Friedler and Christo Wilson, editors, {\em Proceedings
  of the 1st Conference on Fairness, Accountability and Transparency},
  volume~81 of {\em Proceedings of Machine Learning Research}, pages 77--91,
  New York, NY, USA, 23--24 Feb 2018. PMLR.

\bibitem{gardner2019evaluating}
Josh Gardner, Christopher Brooks, and Ryan Baker.
\newblock Evaluating the fairness of predictive student models through slicing
  analysis.
\newblock In {\em Proceedings of the 9th International Conference on Learning
  Analytics \& Knowledge}, pages 225--234. ACM, 2019.

\bibitem{pfohl2019creating}
Stephen Pfohl, Ben Marafino, Adrien Coulet, Fatima Rodriguez, Latha
  Palaniappan, and Nigam~H Shah.
\newblock Creating fair models of atherosclerotic cardiovascular disease risk.
\newblock In {\em Proceedings of the 2019 AAAI/ACM Conference on AI, Ethics,
  and Society}, pages 271--278. ACM, 2019.

\bibitem{madras2018learning}
David Madras, Elliot Creager, Toniann Pitassi, and Richard Zemel.
\newblock Learning adversarially fair and transferable representations, 2018.

\bibitem{Creager2019FlexiblyFR}
Elliot Creager, David Madras, J{\"o}rn-Henrik Jacobsen, Marissa~A. Weis, Kevin
  Swersky, Toniann Pitassi, and Richard~S. Zemel.
\newblock Flexibly fair representation learning by disentanglement.
\newblock {\em ArXiv}, abs/1906.02589, 2019.

\bibitem{edwards2015censoring}
Harrison Edwards and Amos Storkey.
\newblock Censoring representations with an adversary.
\newblock {\em arXiv preprint arXiv:1511.05897}, 2015.

\bibitem{zemel2013learning}
Rich Zemel, Yu~Wu, Kevin Swersky, Toni Pitassi, and Cynthia Dwork.
\newblock Learning fair representations.
\newblock In {\em International Conference on Machine Learning}, pages
  325--333, 2013.

\bibitem{chouldechova2018frontiers}
Alexandra Chouldechova and Aaron Roth.
\newblock The frontiers of fairness in machine learning.
\newblock {\em arXiv preprint arXiv:1810.08810}, 2018.

\bibitem{goldfeld2019convergence}
Ziv Goldfeld, Kristjan Greenewald, Yury Polyanskiy, and Jonathan Weed.
\newblock Convergence of smoothed empirical measures with applications to
  entropy estimation.
\newblock {\em arXiv preprint arXiv:1905.13576}, 2019.

\bibitem{dwork2012fairness}
Cynthia Dwork, Moritz Hardt, Toniann Pitassi, Omer Reingold, and Richard Zemel.
\newblock Fairness through awareness.
\newblock In {\em Proceedings of the 3rd innovations in theoretical computer
  science conference}, pages 214--226, 2012.

\bibitem{agarwal2018reductions}
Alekh Agarwal, Alina Beygelzimer, Miroslav Dud{\'\i}k, John Langford, and Hanna
  Wallach.
\newblock A reductions approach to fair classification.
\newblock {\em arXiv preprint arXiv:1803.02453}, 2018.

\bibitem{kim2018fairness}
Michael Kim, Omer Reingold, and Guy Rothblum.
\newblock Fairness through computationally-bounded awareness.
\newblock In {\em Advances in Neural Information Processing Systems}, pages
  4842--4852, 2018.

\bibitem{kearns2018preventing}
Michael Kearns, Seth Neel, Aaron Roth, and Zhiwei~Steven Wu.
\newblock Preventing fairness gerrymandering: Auditing and learning for
  subgroup fairness.
\newblock In {\em International Conference on Machine Learning}, pages
  2569--2577, 2018.

\bibitem{feldman2015certifying}
Michael Feldman, Sorelle~A Friedler, John Moeller, Carlos Scheidegger, and
  Suresh Venkatasubramanian.
\newblock Certifying and removing disparate impact.
\newblock In {\em Proceedings of the 21th ACM SIGKDD International Conference
  on Knowledge Discovery and Data Mining}, pages 259--268. ACM, 2015.

\bibitem{Gitiaux2019mdfaMF}
Xavier Gitiaux and Huzefa Rangwala.
\newblock mdfa: Multi-differential fairness auditor for black box classifiers.
\newblock In {\em IJCAI}, 2019.

\bibitem{kleinberg2016inherent}
Jon Kleinberg, Sendhil Mullainathan, and Manish Raghavan.
\newblock Inherent trade-offs in the fair determination of risk scores.
\newblock {\em arXiv preprint arXiv:1609.05807}, 2016.

\bibitem{hardt2016equality}
Moritz Hardt, Eric Price, and Nathan Srebro.
\newblock Equality of opportunity in supervised learning, 2016.

\bibitem{calders2013unbiased}
Toon Calders and Indr{\.e} {\v{Z}}liobait{\.e}.
\newblock Why unbiased computational processes can lead to discriminative
  decision procedures.
\newblock In {\em Discrimination and privacy in the information society}, pages
  43--57. Springer, 2013.

\bibitem{gordaliza2019obtaining}
Paula Gordaliza, Eustasio Del~Barrio, Gamboa Fabrice, and Jean-Michel Loubes.
\newblock Obtaining fairness using optimal transport theory.
\newblock In {\em International Conference on Machine Learning}, pages
  2357--2365, 2019.

\bibitem{calmon}
Flavio Calmon, Dennis Wei, Bhanukiran Vinzamuri, Karthikeyan
  Natesan~Ramamurthy, and Kush~R Varshney.
\newblock Optimized pre-processing for discrimination prevention.
\newblock In I.~Guyon, U.~V. Luxburg, S.~Bengio, H.~Wallach, R.~Fergus,
  S.~Vishwanathan, and R.~Garnett, editors, {\em Advances in Neural Information
  Processing Systems 30}, pages 3992--4001. Curran Associates, Inc., 2017.

\bibitem{kurach2018largescale}
Karol Kurach, Mario Lucic, Xiaohua Zhai, Marcin Michalski, and Sylvain Gelly.
\newblock A large-scale study on regularization and normalization in gans,
  2018.

\bibitem{ganin2016domain}
Yaroslav Ganin, Evgeniya Ustinova, Hana Ajakan, Pascal Germain, Hugo
  Larochelle, Fran{\c{c}}ois Laviolette, Mario Marchand, and Victor Lempitsky.
\newblock Domain-adversarial training of neural networks.
\newblock {\em The Journal of Machine Learning Research}, 17(1):2096--2030,
  2016.

\bibitem{zhang2018mitigating}
Brian~Hu Zhang, Blake Lemoine, and Margaret Mitchell.
\newblock Mitigating unwanted biases with adversarial learning.
\newblock In {\em Proceedings of the 2018 AAAI/ACM Conference on AI, Ethics,
  and Society}, pages 335--340, 2018.

\bibitem{xu2018fairgan}
Depeng Xu, Shuhan Yuan, Lu~Zhang, and Xintao Wu.
\newblock Fairgan: Fairness-aware generative adversarial networks.
\newblock In {\em 2018 IEEE International Conference on Big Data (Big Data)},
  pages 570--575. IEEE, 2018.

\bibitem{louizos2015variational}
Christos Louizos, Kevin Swersky, Yujia Li, Max Welling, and Richard Zemel.
\newblock The variational fair autoencoder, 2015.

\bibitem{oneto2019learning}
Luca Oneto, Michele Donini, Andreas Maurer, and Massimiliano Pontil.
\newblock Learning fair and transferable representations, 2019.

\bibitem{dwork2014algorithmic}
Cynthia Dwork, Aaron Roth, et~al.
\newblock The algorithmic foundations of differential privacy.
\newblock {\em Foundations and Trends{\textregistered} in Theoretical Computer
  Science}, 9(3--4):211--407, 2014.

\bibitem{devroye2013probabilistic}
Luc Devroye, L{\'a}szl{\'o} Gy{\"o}rfi, and G{\'a}bor Lugosi.
\newblock {\em A probabilistic theory of pattern recognition}, volume~31.
\newblock Springer Science \& Business Media, 2013.

\bibitem{zhao2013beyond}
Ming-Jie Zhao, Narayanan Edakunni, Adam Pocock, and Gavin Brown.
\newblock Beyond fano's inequality: bounds on the optimal f-score, ber, and
  cost-sensitive risk and their implications.
\newblock {\em Journal of Machine Learning Research}, 14(Apr):1033--1090, 2013.

\bibitem{dsprites17}
Loic Matthey, Irina Higgins, Demis Hassabis, and Alexander Lerchner.
\newblock dsprites: Disentanglement testing sprites dataset.
\newblock https://github.com/deepmind/dsprites-dataset/, 2017.

\bibitem{cover2012elements}
Thomas~M Cover and Joy~A Thomas.
\newblock {\em Elements of information theory}.
\newblock John Wiley \& Sons, 2012.

\end{thebibliography}
\bibliographystyle{unsrt}

\end{document}